\documentclass[journal]{IEEEtran}
\usepackage[T1]{fontenc}
\usepackage{amsmath, amssymb, amsfonts}
\usepackage{graphicx}
\usepackage{algorithm}
\usepackage{algpseudocode}
\usepackage{cite}
\usepackage{hyperref}
\usepackage{booktabs}
\usepackage{tabularx}
\usepackage{array} %
\usepackage{makecell}

\newcolumntype{C}{>{\centering\arraybackslash}X}
\hypersetup{
    colorlinks=true,
    linkcolor=blue,
    filecolor=magenta,      
    urlcolor=cyan,
}

\usepackage{amsthm}

\newtheorem{theorem}{Theorem}[section]
\newtheorem{proposition}{Proposition}[section]
\newtheorem{lemma}{Lemma}[section]

\begin{document}

\title{Statistical-Geometric Degeneracy in UAV Search: A Physics-Aware Asymmetric Filtering Approach}

\author{Zhiyuan~Ren,~\IEEEmembership{Member,~IEEE},
        Yudong~Fang,
        Tao~Zhang,
        Wenchi~Cheng,~\IEEEmembership{Senior~Member,~IEEE},
        Ben~Lan
\thanks{This work was supported by the National Key Research and Development Program of China (No.2023YFC3011502).}%
\thanks{Zhiyuan Ren, Tao Zhang and Wenchi Cheng are with the School of Telecommunications Engineering, Xidian University, Xi'an 710071, China.}
\thanks{Yudong Fang is with Ministry of Emergency Management Big Data Center, Beijing 100013, China}%
\thanks{Ben Lan is with the Guangdong Nasasi Communications Technology Co., Ltd., Guangdong, China.}%
\thanks{Corresponding author: Yudong Fang (fangyudong9713@ustc.edu.cn).}%
}

\maketitle

\begin{abstract}
Post-disaster survivor localization using Unmanned Aerial Vehicles (UAVs) faces a fundamental physical challenge: the prevalence of Non-Line-of-Sight (NLOS) propagation in collapsed structures. Unlike standard Gaussian noise, signal reflection from debris introduces strictly non-negative ranging biases. Existing robust estimators, typically designed with symmetric loss functions (e.g., Huber or Tukey), implicitly rely on the assumption of error symmetry. Consequently, they experience a theoretical mismatch in this regime, leading to a phenomenon we formally identify as \textit{Statistical-Geometric Degeneracy} (SGD)—a state where the estimator stagnates due to the coupling of persistent asymmetric bias and limited observation geometry. While emerging data-driven approaches offer alternatives, they often struggle with the scarcity of training data and the sim-to-real gap inherent in unstructured disaster zones. In this work, we propose a physically-grounded solution, the \textit{AsymmetricHuberEKF}, which explicitly incorporates the non-negative physical prior of NLOS biases via a derived asymmetric loss function. Theoretically, we show that standard symmetric filters correspond to a degenerate case of our framework where the physical constraint is relaxed. Furthermore, we demonstrate that resolving SGD requires not just a robust filter, but specific \textit{bilateral information}, which we achieve through a co-designed active sensing strategy. Validated in a 2D nadir-view scanning scenario, our approach significantly accelerates convergence compared to symmetric baselines, offering a resilient building block for search operations where data is scarce and geometry is constrained.
\end{abstract}


\section{Introduction}
\label{sec:intro}

\IEEEPARstart{I}n the aftermath of natural disasters, the rapid restoration of Public Safety Networks (PSNs) and the precise localization of survivors are paramount. Unmanned Aerial Vehicles (UAVs) have emerged as agile aerial base stations, capable of executing nadir-view scanning missions over inaccessible areas to detect User Equipment (UE) signals \cite{albanese_sardo_2022}. However, this operational setting presents a fundamental physical challenge that distinguishes it from standard open-field navigation: the prevalence of extreme Non-Line-of-Sight (NLOS) propagation within collapsed structures or urban canyons. Unlike the zero-mean Gaussian noise typical of Line-of-Sight (LOS) conditions, signal reflection and diffraction caused by rubble obstacles introduce strictly non-negative biases in range-based measurements \cite{yu_statistical_2009}. This physical constraint means that the observed distance is consistently equal to or greater than the true geometric distance, creating a structured error pattern that fundamentally violates the symmetry assumptions inherent in classical estimation theory.

\begin{figure}[!t]
    \centering
    \includegraphics[width=1\columnwidth]{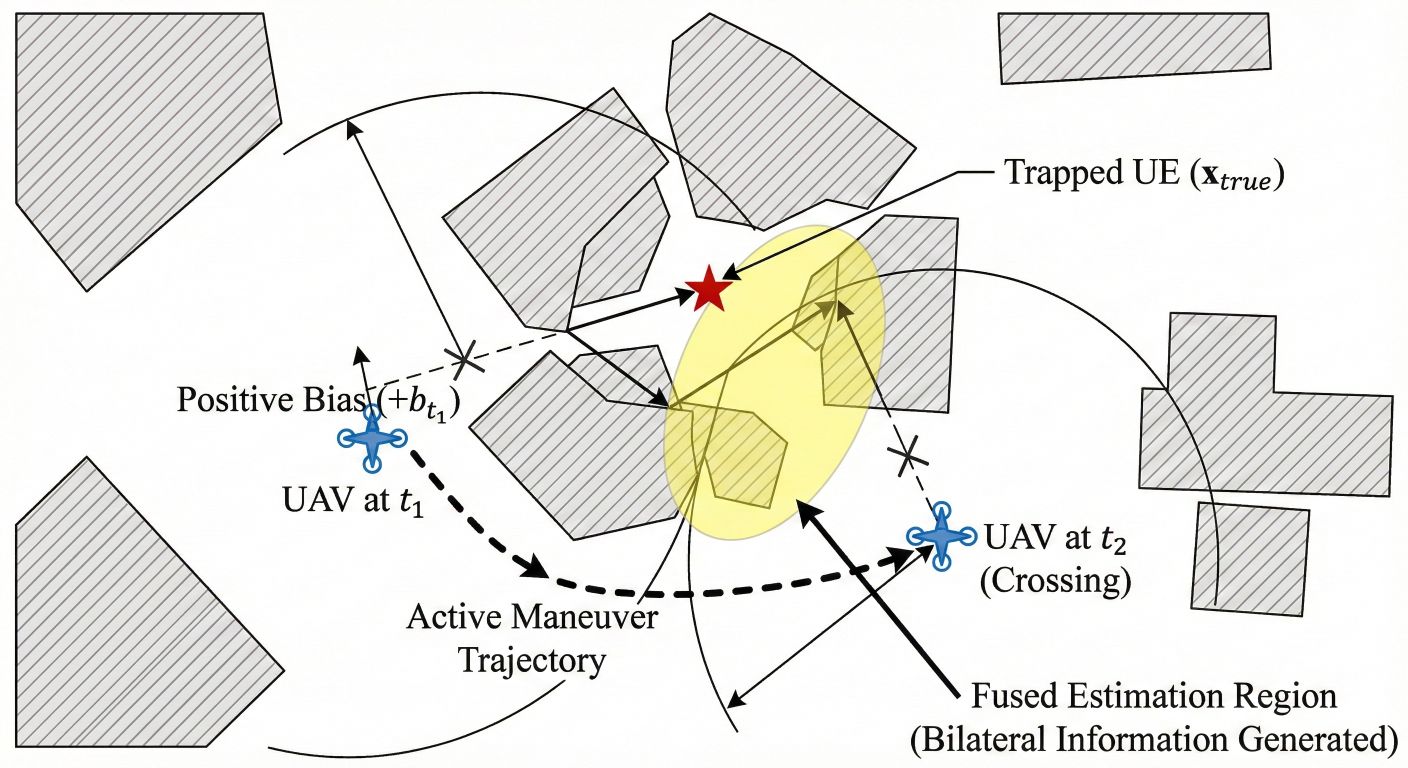} 
    \caption{Schematic illustration of the NLOS localization problem and the proposed active sensing solution in a post-disaster scenario.
The UAV acts as an aerial base station trying to localize a trapped User Equipment (UE). Grey polygons represent rubble obstacles blocking the Line-of-Sight (LOS). At time $t_1$, the NLOS signal reflection causes a significantly positively biased range measurement (outer arc). By actively maneuvering to $t_2$ (a "crossing" geometry), the UAV obtains a second, differently biased measurement. The intersection of these biased geometric constraints generates "bilateral information," significantly constraining the estimated target region despite the persistent NLOS conditions.}
    \label{fig:nlos_effect}
\end{figure}

\begin{table*}[t]
\centering
\caption{Comparison of Approaches: Assumptions and Failure Modes in Post-Disaster Networks.}
\label{tab:comparison}
\begin{tabular}{l c c c}
\toprule
Method Class & Core Assumption & Mechanism & Failure Mode in Post-Disaster Regime \\
\midrule
Standard EKF \cite{he_new_2025, yu_statistical_2009} & Zero-mean Gaussian & Linearization/Rejection & \textbf{Divergence}: Violates non-negative physical constraint. \\
Symmetric Robust \cite{ho_robust_2022, wang_robust_2023} & Symmetric Outliers & M-Estimation & \textbf{Stagnation (SGD)}: Cannot distinguish positive bias from geometry. \\
Data-Driven \cite{Wang2024Robust, sun_optimal_2024} & i.i.d. Training Data & Black-box Learning & \textbf{Generalization Gap}: Suffers from data scarcity in ruins. \\
\textbf{Ours (Proposed)} & \textbf{Physical Non-negativity} & \textbf{Asymmetric Physics} & \textbf{Robust \& Fast}: Resolves degeneracy via bilateral information. \\
\bottomrule
\end{tabular}
\end{table*}

Prior research has largely focused on mitigating these errors through statistical rejection or symmetric re-weighting. However, we identify a critical gap in how these methods model the underlying error structure in post-disaster scenarios. We categorize existing solutions into three classes, as summarized in Table~\ref{tab:comparison}. 

\textbf{Class I (Standard Estimation)} assumes Gaussian noise or relies on hypothesis testing to reject NLOS \cite{he_new_2025, yu_statistical_2009}, leading to immediate divergence or information loss when the mean error is non-zero. 

\textbf{Class II (Symmetric Robust Filters)}, such as those using Huber or Tukey loss functions \cite{ho_robust_2022, wang_robust_2023}, implicitly assume that large errors are symmetric outliers. While they effectively suppress bidirectional glitches, they struggle to distinguish persistent, one-sided physical biases from valid geometric information, leading to the stagnation phenomenon we explore in this work. 

\textbf{Class III (Data-Driven Approaches)} leverage Deep Learning or Reinforcement Learning to learn complex error mappings \cite{Wang2024Robust, sun_optimal_2024}. While promising, their practical deployment in unstructured ruins is often hindered by the fundamental scarcity of training data and the significant sim-to-real gap. 

Consequently, a robust, model-based approach that intrinsically respects the non-negative physical constraint—without relying on extensive pre-training—remains a missing piece in the public safety network literature.

We observe that the limitation of existing robust methods is not accidental but \textit{structural}. It stems from the decoupling of the estimator's loss function from the channel's physical constraints. We argue that to rigorously resolve the localization problem in rubble-filled environments, one must explicitly model the \textit{asymmetry} of the underlying error process. Building on this insight, we derive the \textit{AsymmetricHuberEKF} within a Bayesian framework. By analytically marginalizing the non-negative NLOS bias, we obtain a novel "one-sided" robust cost function. Crucially, this formulation reveals a unifying perspective: the standard symmetric Huber loss corresponds to a relaxed case of our model where the non-negativity constraint is ignored. Furthermore, our analysis exposes that breaking the Statistical-Geometric Degeneracy requires more than just a robust filter; it necessitates specific \textit{bilateral information}—geometric diversity that can only be guaranteed through active UAV maneuvering \cite{lin_wang_active_2010, kwon_rf_2023}, rather than passive observation.

The main contributions of this work are summarized as follows:
\begin{itemize}
    \item \textbf{Formalization of Statistical-Geometric Degeneracy (SGD):} We formally define SGD within the context of mobile positioning. We establish a theoretical framework that explains the root cause of estimation stagnation as the coupling between asymmetric error statistics and limited one-sided sensing geometries.
    \item \textbf{Unified Robust Modeling:} We derive the \textit{AsymmetricHuberEKF} from a Bayesian Maximum A Posteriori (MAP) perspective. We prove that this physically-grounded estimator acts as a unified framework, where the standard symmetric Huber filter appears as a special case when the non-negative physical constraint is relaxed.
    \item \textbf{Synergistic Active Sensing:} We identify that resolving SGD requires specific \textit{bilateral information}. We demonstrate that a simple geometric crossing maneuver serves as a sufficient condition to generate this information and restore system observability, thereby validating the theoretical link between physical priors and UAV motion control.
\end{itemize}

\subsection*{Scope and Non-Goals}
This work focuses on the \emph{mechanism} foundations of robust localization in post-disaster environments. 
\textbf{Scope:} We consider a 2D nadir-view scanning model where the UAV altitude is assumed known. This abstraction allows us to rigorously isolate and analyze the interaction between the estimator and the sensing geometry.
\textbf{Non-Goals:} We do not address 3D Simultaneous Localization and Mapping (SLAM) or complex flight dynamics control. These system-level engineering challenges are orthogonal to the algorithmic mechanism proposed in this study.

The remainder of this paper is organized to systematically unfold this theoretical framework. 
Section~\ref{sec:related_work} reviews related work and identifies the research gap. 
Section~\ref{sec:system} establishes the system model and formally defines the \textit{statistical-geometric degeneracy}. 
Section~\ref{sec:estimation} derives the physically-grounded \textbf{AsymmetricHuberEKF}, while Section~\ref{sec:planners} details the geometric planning strategies designed to generate the necessary bilateral information. 
Section~\ref{sec:simulations} presents comprehensive simulations to validate the proposed approach. Finally, Section~\ref{sec:limitations} discusses the theoretical boundaries and modeling assumptions, and Section~\ref{sec:conclusion} concludes the paper.

\section{Related Work}
\label{sec:related_work}
Recent literature in Vehicular Technology has prioritized the upper-layer optimization of Public Safety Networks (PSNs).
For instance, Qin et al.~\cite{qin_multi-type_2025} and Sun et al.~\cite{sun_optimal_2024} have developed sophisticated frameworks for task offloading and resource allocation to enhance rescue efficiency.
While vital, these works optimize \textit{service delivery}, contingent on the assumption that stable links with survivors exist.
Our work addresses the more fundamental challenge of \textit{service enablement}: robustly localizing survivors in denied environments to physically establish the connection.

This problem sits at the challenging intersection of robust NLOS estimation and active UAV sensing \cite{albanese_sardo_2022}.
However, existing research often treats these domains as orthogonal problems.
Robust filters are typically designed in isolation assuming fixed geometries, while active planning strategies \cite{lin_wang_active_2010, kwon_rf_2023} usually rely on idealized Gaussian noise models to pursue generic objectives.
We argue that this \textit{structural decoupling} creates a blind spot: standard planners fail to break the specific geometric symmetries that confuse robust filters.
Consequently, our work proposes a deep, synergistic co-design to resolve the \textit{Statistical-Geometric Degeneracy} identified in Section I.

\subsection{NLOS Mitigation and Robust Estimation}

\subsubsection{NLOS Identification and Rejection}
A classical paradigm in NLOS mitigation is to treat non-line-of-sight measurements as statistical outliers to be identified and discarded. Techniques range from hypothesis testing based on signal statistics \cite{he_new_2025, yu_statistical_2009} to machine learning classifiers \cite{Wang2024Robust}. For instance, Yu and Guo \cite{yu_statistical_2009} utilized statistical decision theory based on AOA and TOA to identify NLOS conditions. While effective in mixed environments, a primary limitation of this "rejection-based" paradigm is the loss of geometric information contained within the NLOS measurements \cite{zheng2024location, zheng2025location}. In post-disaster scenarios like rubble or urban canyons, where NLOS signals may be the dominant or even the sole source of information \cite{albanese_sardo_2022}, discarding these measurements often leads to an ill-posed estimation problem. Furthermore, while data-driven classifiers show promise, they face significant challenges in dynamic SAR scenarios due to the scarcity of training data for unstructured environments. In contrast, our approach seeks to \textit{utilize} rather than discard these biased measurements by modeling their physical properties.

\subsubsection{Symmetric Robust Estimation}
To retain measurement information, a widely adopted approach in vehicular navigation is M-estimation, which down-weights large residuals using robust loss functions like Huber or Tukey~\cite{ho_robust_2022}. This principle has been integrated into various filtering frameworks, including the Unscented Kalman Filter (UKF)~\cite{ho_enhanced_2021}, Iterated Cubature Kalman Filter (ICKF) for GNSS/INS integration~\cite{wang_robust_2023}, and Particle Filters for dynamic state estimation~\cite{ho_mobile_2021}. Additionally, the Interacting Multiple Model (IMM) framework can run robust and standard filters in parallel to handle maneuvering targets~\cite{hammes_robust_2011}. However, a fundamental theoretical limitation of these methods is their \textbf{symmetric nature}. By design, standard M-estimators satisfy $\rho(r) = \rho(-r)$, implicitly assuming that outliers are equally likely to be positive or negative. This assumption creates a \textit{theoretical mismatch} in wireless ranging, where the physical NLOS bias is strictly non-negative. Consequently, while these symmetric filters can mitigate sporadic sensor glitches, they often stagnate when facing the persistent, one-sided biases inherent to radio propagation in urban canyons, failing to resolve the geometric ambiguity.

\subsubsection{Advanced Modeling and Optimization}
Beyond symmetric filtering, more sophisticated strategies explicitly model the NLOS bias using heavy-tailed distributions~\cite{ciosas_nlos_2016, zhang_adaptive_2025} or leverage waveform-specific properties in Joint Communication and Sensing (JCAS)~\cite{nordio_joint_2025}. While precise, these methods often require specific physical layer modifications.

A more recent direction involves formulating the estimation as a constrained optimization problem. Methods based on the Alternating Direction Method of Multipliers (ADMM), for instance, can explicitly enforce the non-negativity of the NLOS bias by splitting the problem into simpler sub-problems~\cite{Sahu2022ADMM}. While powerful, these optimization-based methods introduce significant computational overhead and tuning complexities, which can be prohibitive for real-time UAV applications.

In contrast, our work builds upon a similar physical intuition regarding non-negativity but derives a solution via a statistically-motivated marginalization. By systematically deriving a protocol-agnostic \textbf{``one-sided Huber''} framework from a Bayesian MAP formulation, we offer a closed-form, computationally efficient solution. This provides a lightweight alternative for mobile nodes that enforces physical constraints without the complexity of iterative optimization.

\subsection{Active Sensing for Mobile Service Recovery}
Active sensing transforms the UAV from a passive observer into an intelligent information seeker, a concept widely adopted in robotics~\cite{Heng1997Active} and increasingly critical for autonomous mobile network nodes.

\subsubsection{Information-Theoretic Approaches in Vehicular Networks}
The dominant paradigm in this domain relies on information theory, where motion planning seeks to maximize a metric related to the Fisher Information Matrix (FIM)~\cite{La2015Cooperative, Napolitano2024ActiveFlat}. In the context of vehicular technology, Lin et al.~\cite{lin_wang_active_2010} applied this principle to cooperative target tracking, demonstrating how fusion estimation can guide UAV trajectories. More recently, Kwon and Guvenc~\cite{kwon_rf_2023} explored RF signal source localization using UAVs, highlighting the trade-off between accuracy and flight distance in resource-constrained systems.

A critical limitation, however, is that these frameworks typically assume a well-behaved measurement model (e.g., zero-mean Gaussian) and pursue \textbf{generic uncertainty reduction}. While effective in line-of-sight scenarios, they do not account for the specific \textit{statistical-geometric degeneracy} induced by asymmetric NLOS bias. Consequently, generic FIM-based objectives often fail to prioritize the specific ``crossing'' maneuvers required to recover system observability in denied environments, as they cannot distinguish between geometric poor conditioning and bias-induced uncertainty.

\subsubsection{Learning-based vs. Model-based Active Positioning}

While Reinforcement Learning (RL) has emerged as a powerful framework for learning complex motion policies that balance exploration and exploitation~\cite{Napolitano2024ActiveData, Yang2023Learning, Bartolomei2021Semantic, Selimovic2024MultiAgent}, its core requirements present critical hurdles for the practical demands of public safety networks. For instance, Wu et al.~\cite{wu_adaptive_2023} proposed an adaptive Q-Learning approach for multi-mission SAR path planning, and Sun et al.~\cite{sun_optimal_2024} utilized Double Deep Q-Learning for energy-efficient trajectory design.

However, these approaches face a fundamental hurdle in our context. \textbf{While data-driven approaches (e.g., Deep Learning) have shown promise in fixed infrastructure settings, they struggle in dynamic SAR scenarios due to the \textit{scarcity of training data} for unstructured environments and the risk of poor generalization to unseen geometries.} Furthermore, the ``black-box'' nature of learned policies raises reliability concerns in mission-critical mobile services, where unexpected emergent behavior could compromise rescue safety. Additionally, as highlighted in resource-constrained studies~\cite{kwon_rf_2023}, the computational overhead of onboard training is often prohibitive for lightweight aerial base stations.

In contrast, our work proposes a transparent, geometrically grounded motion protocol—the ``crossing-over'' maneuver. Derived from the analytical properties of our filter, it manufactures the ``bilateral information'' required for rapid convergence, offering a verifiable solution that mitigates the data dependency and safety risks inherent to learning-based approaches.

\section{System Model and Problem Formulation}
\label{sec:system}

In this section, we establish a unified mathematical framework to describe the process of a single mobile agent localizing a static target. We adopt a constructive approach: first defining the geometric observation setup, then modeling the specific physical statistics of signal propagation in rubble, and finally formulating the localization task as a constrained Bayesian inference problem.

\subsection{Geometric Setup and Scope Abstraction}
We consider a 2-dimensional search scenario, which represents the \textit{nadir-view scanning paradigm} prevalent in post-disaster operations. In this model, the UAV maintains a known safety altitude $H$ while cruising over the affected area. Consequently, the 3D localization problem is mathematically projected onto the 2D flight plane. Let $x \in \mathbb{R}^2$ denote the unknown horizontal position of the survivor's User Equipment (UE). The mobile agent's pose sequence is denoted by $\{s_t\}_{t=1}^T$, where $s_t \in \mathbb{R}^2$.

\textbf{Assumption 1 (Pose Knowledge):} We assume the agent's pose $s_t$ is known. This abstraction is intentional to strictly decouple the \textit{sensing geometry} problem (the focus of this work) from the \textit{navigation drift} problem (SLAM). It allows us to isolate the specific impact of NLOS biases on estimation stability.

\subsection{The Physical Generative Model}
At each time step $t$, the agent acquires measurements from a set of modalities $\mathcal{M}$ (e.g., RTT, AoA). We formulate a general generative observation model that explicitly accounts for both systematic hardware offsets and environment-induced physical biases:
\begin{equation}
y_{m,t} = h_m(x; s_t) + \delta_m + b_{m,t} + \varepsilon_{m,t}
\label{eq:measurement_model}
\end{equation}
where:
\begin{itemize}
    \item $h_m(x; s_t)$ is the ideal geometric observation function (e.g., Euclidean distance for RTT, $\arctan$ for AoA).
    \item $\delta_m$ represents the time-invariant \textit{systematic bias} (e.g., clock synchronization offset or antenna orientation error). This is a signed value, $\delta_m \in \mathbb{R}$.
    \item $\varepsilon_{m,t} \sim \mathcal{N}(0, \sigma_m^2)$ is the thermal measurement noise.
    \item $b_{m,t}$ is the \textit{environment-induced NLOS bias}, which exhibits distinct physical characteristics depending on the modality:
\end{itemize}

\subsubsection{Asymmetric Physics of RTT}
For Round-Trip Time (RTT) measurements, $b_{m,t}$ represents the excess path length caused by signal reflection. A fundamental physical law in this context is that the signal path cannot be shorter than the straight line. Thus, the bias is governed by a \textbf{Strict Non-Negativity Constraint}:
\begin{equation}
    b_{m,t} \ge 0, \quad \forall t, m \in \mathcal{M}_{\text{RTT}}
\end{equation}
We model this bias as a latent variable drawn from a one-sided distribution (e.g., Exponential), $b_{m,t} \sim p_{\text{NLOS}}(b; \lambda)$.

\subsubsection{Symmetric Physics of AoA}
For Angle of Arrival (AoA), signal scattering can cause deviations in either clockwise or counter-clockwise directions. Thus, the bias $b_{m,t}$ is modeled as a symmetric outlier process with no sign constraint, typically following a heavy-tailed distribution.

\subsection{Unified Problem Formulation}
Given the measurement history $\mathcal{Y}_T$, our goal is to jointly estimate the target position $x$ and the systematic biases $\{\delta_m\}$. From a Bayesian perspective, this corresponds to finding the Maximum A Posteriori (MAP) estimate.
By incorporating the physical priors defined above, the objective function $J(x, \{\delta_m\})$ generally takes the form:
\begin{equation}
    J = \underbrace{\sum_{t} \mathcal{L}_{\text{fit}}(y_t | x, \delta_m, b_t)}_{\text{Data Fidelity}} + \underbrace{\sum_{t} \mathcal{R}_{\text{phy}}(b_t)}_{\text{Physical Prior}}
\end{equation}
The specific derivation of $\mathcal{R}_{\text{phy}}$ leads to different robust estimators.

It is crucial to note that standard robust approaches (e.g., Huber or Tukey filters) can be viewed as a \textit{special case} of this formulation where the non-negative physical constraint on RTT bias is \textbf{relaxed} (i.e., assuming $b_{m,t} \in \mathbb{R}$ and symmetric). Our framework differs by rigorously retaining this constraint, ensuring that the estimator structure remains consistent with the channel physics.

\section{The Proposed Robust Estimation Framework}
\label{sec:estimation}
Building upon the system model, this section details the robust fusion estimation algorithm designed to solve the localization problem. The core of this framework is to formulate the problem within a Bayesian Maximum a Posteriori (MAP) paradigm and, from this, systematically derive a robust cost function that intrinsically and asymmetrically handles the positive bias characteristic of NLOS propagation.

\subsection{Bayesian Modeling via Maximum a Posteriori (MAP) Estimation}

To invert the physical generative model established in Section~\ref{sec:system}, we formulate the localization problem as maximizing the posterior probability of the unknown parameters given the measurement history. We seek the MAP estimate for the augmented parameter set, which includes the target position $x$, the modality-specific systematic biases $\{\delta_m\}$, and the latent NLOS biases $\{b_{m,t}\}$ for RTT measurements.

We systematically translate the physical constraints identified in Section III into statistical priors:
\begin{enumerate}
    \item \textbf{Target Position $x$:} We assign a non-informative flat prior $p(x) \propto 1$, ensuring the estimate is driven solely by observational data.
    \item \textbf{Systematic Bias $\delta_m$:} Consistent with hardware calibration errors, we assign a zero-mean Gaussian prior $\delta_m \sim \mathcal{N}(0, \sigma_{\delta,m}^2)$.
    \item \textbf{NLOS Bias $b_{m,t}$ (The Physical Constraint):}
    \begin{itemize}
        \item \textit{RTT (Asymmetric):} To strictly enforce the non-negativity law ($b_{m,t} \ge 0$), we assign a one-sided \textbf{exponential prior} $b_{m,t} \sim \text{Exp}(\lambda_m)\mathbb{I}(b_{m,t} \ge 0)$. This penalizes large biases while forbidding physically impossible negative values.
        \item \textit{AoA (Symmetric):} For angular deviations, we adopt a \textbf{symmetric heavy-tailed prior} (specifically, Huber's least-favorable distribution). This accounts for bidirectional outliers without imposing sign constraints.
    \end{itemize}
\end{enumerate}

Minimizing the negative log-posterior yields the following hybrid objective function:
\begin{equation}
\begin{split}
\min_{\substack{x, \{\delta_m\} \\ \{b_{m,t}\ge0\}_{m \in \mathcal{M}_{\text{RTT}}} }} \bigg( & \underbrace{\sum_{m \in \mathcal{M}_{\text{RTT}}, t} \frac{\big(y_{m,t}-h_m(x;s_t)-\delta_m-b_{m,t}\big)^2}{2\sigma_m^2}}_{\text{RTT Data Fidelity (Explicit NLOS State)}} \\
& + \underbrace{\sum_{m \in \mathcal{M}_{\text{AoA}}, t} \rho_{\text{aoa}}\big(y_{m,t}-h_m(x;s_t)-\delta_m\big)}_{\text{AoA Robust Fidelity (Implicit)}} \\
& + \sum_{m} \frac{\delta_m^2}{2\sigma_{\delta,m}^2} + \sum_{m \in \mathcal{M}_{\text{RTT}}, t} \lambda_m b_{m,t} \bigg)
\end{split}
\label{eq:full_map}
\end{equation}
where $\rho_{\text{aoa}}(\cdot)$ represents the robust Huber loss derived from the AoA prior. This formulation explicitly couples the estimation of the target $x$ with the resolution of the latent NLOS variables $b_{m,t}$, subject to their respective physical laws.

\subsection{The Hybrid Robust Cost Function via Analytical Marginalization}

Direct optimization of \eqref{eq:full_map} is computationally challenging due to the high dimensionality of the latent NLOS variables $\{b_{m,t}\}$. However, we observe that the problem possesses a decoupled structure that allows us to \textbf{analytically marginalize out} these nuisance parameters. This process effectively "compiles" the physical prior directly into a new robust loss function.

\subsubsection{Derivation of the One-sided Huber Loss (RTT)}
For RTT measurements ($m \in \mathcal{M}_{\text{RTT}}$), we isolate the optimization subproblem with respect to each NLOS bias $b_{m,t}$. Let $r_{m,t} = y_{m,t}-h_m(x;s_t)-\delta_m$ denote the raw residual. The local optimization becomes:
\begin{equation}
\min_{b_{m,t}\ge0} \left( \frac{(r_{m,t}-b_{m,t})^2}{2\sigma_m^2} + \lambda_m b_{m,t} \right)
\end{equation}
This is a quadratic program subject to a non-negativity constraint. We derive its exact solution as follows:

\begin{proposition}[Analytical Solution for NLOS Bias]
The optimization subproblem for the RTT NLOS bias $b_{m,t}$ has a unique, closed-form solution given by a soft-thresholding operator:
\begin{equation}
b_{m,t}^* = \max\left(0, \ r_{m,t} - \lambda_m\sigma_m^2\right)
\label{eq:soft_threshold}
\end{equation}
\end{proposition}
\begin{IEEEproof}
    The proof is provided in Appendix~\ref{app:onesided_huber}, Proposition~\ref{prop:kkt_solution}.
\end{IEEEproof}

\begin{figure}[!t]
    \centering
    \includegraphics[width=0.9\columnwidth]{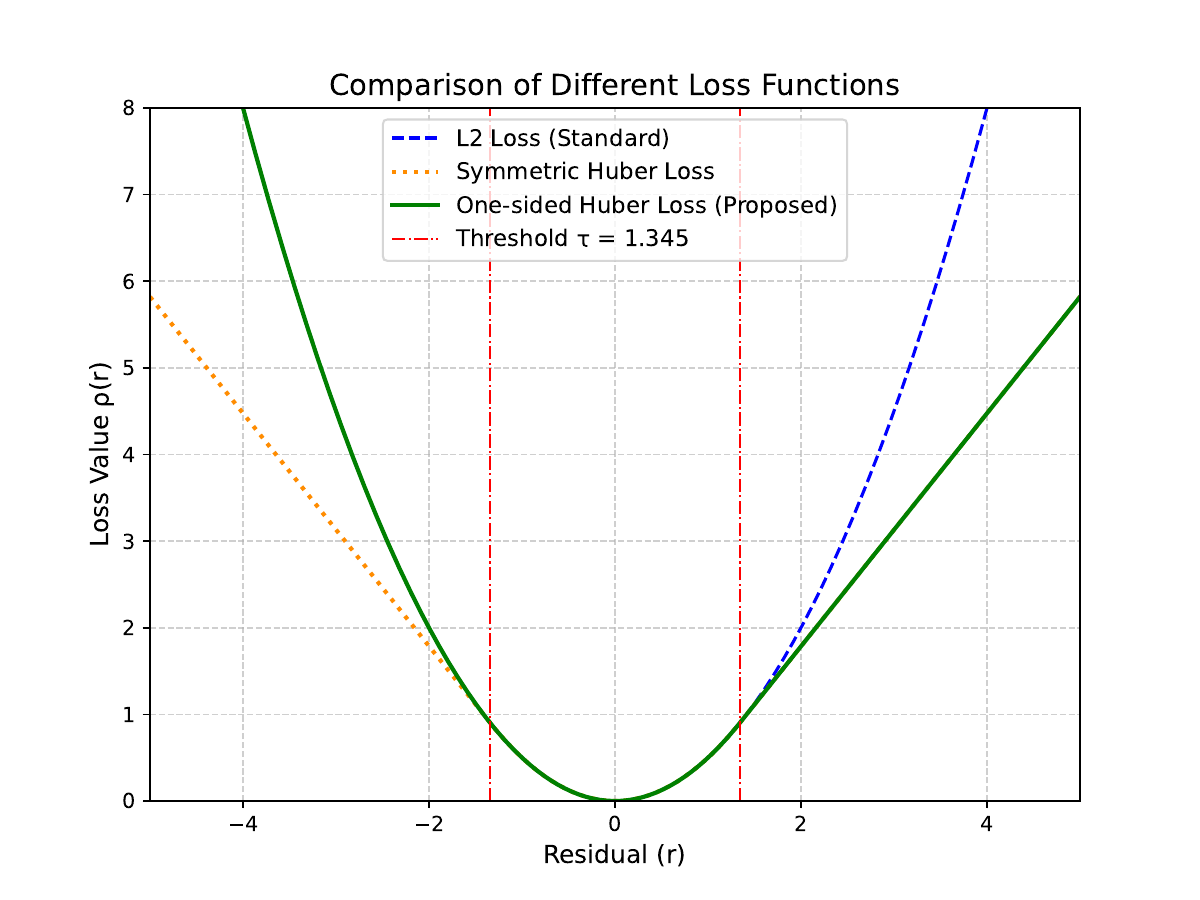} 
    \caption{Visual comparison of different loss functions versus the residual $r$. The standard L2 loss (blue, dashed) penalizes all errors quadratically. The symmetric Huber loss (orange, dotted) is robust to outliers in both directions. Our proposed one-sided Huber loss (green, solid) is asymmetric: it correctly treats large positive residuals (indicative of NLOS) with linear penalty, while penalizing negative residuals quadratically, thus retaining sensitivity to potentially LOS measurements.}
    \label{fig:loss_functions}
\end{figure}

By substituting this optimal $b_{m,t}^*$ back into the original MAP objective, we eliminate the latent variables and obtain a marginalized cost function, $\rho_{\text{rtt}}(r_{m,t})$, which we term the \textbf{``one-sided Huber'' loss}:
\begin{equation}
\rho_{\text{rtt}}(r) = 
\begin{cases} 
  \frac{1}{2\sigma^2} r^2 & \text{if } r \le \tau \\
  \lambda r - \frac{1}{2}\lambda^2\sigma^2 & \text{if } r > \tau
\end{cases}
\label{eq:one_sided_huber}
\end{equation}
where the threshold is physically determined by $\tau = \lambda\sigma^2$. This asymmetry (visualized in Fig.~\ref{fig:loss_functions}) is the direct mathematical consequence of the non-negative physical prior.

\begin{proposition}[Properties of the One-sided Huber Loss]
The marginalized cost function $\rho_{\text{rtt}}(r)$ defined in \eqref{eq:one_sided_huber} is convex and continuously differentiable ($C^1$) everywhere.
\end{proposition}
\begin{IEEEproof}
    The derivation is given in Proposition~\ref{prop:marginalized_cost}, and the proof of convexity and $C^1$ smoothness is detailed in Proposition~\ref{prop:huber_properties} in Appendix~\ref{app:onesided_huber}.
\end{IEEEproof}

\subsubsection{Symmetric Huber Loss (AoA)}
For AoA measurements, consistent with the symmetric heavy-tailed prior, the standard \textbf{symmetric Huber loss} is naturally recovered:
\begin{equation}
\rho_{\text{aoa}}(r) = 
\begin{cases} 
  \frac{1}{2\sigma^2} r^2 & \text{if } |r| \le \tau \\
  \frac{k}{\sigma} |r| - \frac{1}{2}k^2 & \text{if } |r| > \tau
\end{cases}
\label{eq:symmetric_huber}
\end{equation}
where $\tau = k\sigma$.

\subsubsection{The Final Unified Objective}
Combining the asymmetric (RTT) and symmetric (AoA) components, the final reduced-dimension optimization problem becomes:
\begin{equation}
\min_{x, \{\delta_m\}} \bigg( \sum_{m \in \mathcal{M}_{\text{RTT}}, t} \rho_{\text{rtt}}(r_{m,t}) + \sum_{m \in \mathcal{M}_{\text{AoA}}, t} \rho_{\text{aoa}}(r_{m,t}) + \sum_m \frac{\delta_m^2}{2\sigma_{\delta,m}^2} \bigg)
\label{eq:final_objective}
\end{equation}
This formulation represents a \textbf{physically-grounded fusion}: it applies linear penalties to large positive RTT errors (NLOS) while maintaining quadratic sensitivity to negative residuals (likely LOS or noise), effectively resolving the "theoretical mismatch" identified in Section I.

\subsection{Parameter Equivalence and Implementation Strategy}
Our hybrid model introduces physical hyperparameters ($\lambda_m$ for RTT) that differ from the standard statistical tuning constants ($k_m$) found in generic robust libraries. To ensure practical deployability, we rigorously establish the equivalence between these two representations.

\subsubsection{Bridging Physics and Implementation}
The threshold for the one-sided RTT cost function, derived in \eqref{eq:soft_threshold}, is physically determined by $\tau_m = \lambda_m \sigma_m^2$. Here, $1/\lambda_m$ represents the expected physical magnitude of the NLOS bias (from the exponential prior). Conversely, standard implementations typically define a threshold on the \textit{normalized} residual ($r/\sigma$) using a unitless constant $k_m$.
To align our theoretical derivation with standard engineering practice, we equate the effective thresholds:
\begin{equation}
\tau_m = \lambda_m \sigma_m^2 = k_m \sigma_m \implies k_m = \lambda_m \sigma_m
\label{eq:k_lambda_relation}
\end{equation}
This relationship, formally proven in Proposition~\ref{prop:param_equivalence} (see Appendix~\ref{app:onesided_huber}), serves as a translation layer. It clarifies that choosing a "robust tuning constant" $k_m$ is mathematically equivalent to positing a specific physical NLOS scale $\lambda_m$ relative to the noise floor $\sigma_m$. For instance, a smaller $k_m$ implies a smaller $\lambda_m$ (larger expected bias), forcing the filter to treat a wider range of residuals as linear (NLOS) rather than quadratic (LOS).

\subsubsection{Path to Adaptive Estimation}
While we utilize fixed hyperparameters for validation, our framework naturally extends to adaptive estimation in unknown environments. The proposed probabilistic model allows for the online learning of $\lambda_m$ (and thus $k_m$) via an Expectation-Maximization (EM) approach. Specifically:
\begin{itemize}
    \item \textbf{E-Step:} Estimate the latent NLOS biases $b_{m,t}^*$ given the current parameter guess.
    \item \textbf{M-Step:} Update $\lambda_m$ by taking the inverse of the sample mean of these estimated biases, $\lambda_m^{new} = (\frac{1}{N}\sum b_{m,t}^*)^{-1}$.
\end{itemize}
This adaptability is essential for future deployment in dynamic disaster zones where the severity of NLOS ($\lambda_m$) varies spatially.

\subsection{Geometric Conditioning and the Necessity of Bilateral Information}

While the derived one-sided Huber loss is statistically optimal for asymmetric noise, its sequential performance is intrinsically coupled with the quality of the sensing geometry. To formalize this estimator-geometry interaction, we analyze the problem's conditioning using the \textbf{Expected Robust Curvature Matrix (ERCM)}.

The RTT component of the ERCM, denoted $\mathcal{C}_{\text{RTT}}(x)$, describes the aggregate curvature (second-order information) contributed by the range measurements. As formally derived in Lemma~\ref{lem:ercm_form} (Appendix~\ref{app:ercm_proofs}), it takes the form:
\begin{equation}
\mathcal{C}_{\text{RTT}}(x) \approx \frac{1}{\sigma_r^2}\sum_{t \in \mathcal{T}_{\text{lin}}} H_t H_t^\top
\end{equation}
where $\mathcal{T}_{\text{lin}} = \{t \mid r_t(x) \le \tau_t\}$ is the subset of time steps where the residual falls within the quadratic (linear) region. 

\textbf{Mechanism of Degeneracy.} Crucially, the second derivative of our robust loss vanishes ($\rho''_{\text{rtt}} = 0$) in the saturation region ($r > \tau$). This implies that the curvature of the loss landscape is constructed \textit{exclusively} from measurements that produce small or non-positive residuals. This structural property allows us to formally define \textbf{Bilateral Information}: a geometric state where the observation sequence contains a sufficiently diverse subset of non-saturated samples to ensure the ERCM is positive definite. As proven in Proposition~\ref{prop:strong_convexity} (Appendix~\ref{app:ercm_proofs}), the presence of bilateral information endows the objective function with \textit{restricted strong convexity}, a prerequisite for fast convergence.

\textbf{Theoretical Necessity of Active Sensing.} This analysis reveals the fundamental vulnerability of passive estimation. If a passive agent's trajectory relative to the target produces a persistent sequence of large, positive residuals (unilateral information), all samples effectively become saturated. In this scenario, as stated in Proposition~\ref{prop:necessity} (Appendix~\ref{app:ercm_proofs}), the ERCM becomes rank-deficient or even zero. This creates a flat "ambiguity valley" in the optimization landscape where second-order solvers stagnate. Consequently, within our framework, acquiring bilateral information is not merely beneficial; it is a \textbf{theoretical necessity} to ensure the problem remains well-posed. This rigid requirement directly motivates the co-design of the active motion planning strategies detailed in the next section.

\section{Motion Planning Strategies for Validation}
\label{sec:planners}

The theoretical analysis in Section~\ref{sec:estimation}-D established that resolving the statistical-geometric degeneracy requires the active generation of "bilateral information." To validate the efficacy of our proposed \textbf{AsymmetricHuberEKF} in this context, we select and implement two representative active planning strategies. These planners, one embodying computational efficiency and the other information-theoretic optimality, allow us to verify whether the "bilateral" condition identified in our theory is indeed the key to restoring observability.

\subsection{Heuristic-based Planner: Reactive Crossing}
The first strategy is a lightweight, heuristic-based planner we term \textbf{Reactive Crossing}. Its logic, presented in Algorithm~\ref{alg:reactive_crossing}, is geometrically intuitive: at each time step, it commands the UAV to fly towards a point situated on the far side of its current target estimate.

While this algorithm is structurally simple, we employ it here because its mechanism directly serves the theoretical goal of manufacturing bilateral information. As proven in Appendix~\ref{app:crossing_proofs}, such a crossing maneuver is a sufficient condition to improve the conditioning of the estimation problem (ERCM), provided a non-saturated residual can be generated. Its key advantage lies in its constant-time complexity ($O(1)$), making it an ideal candidate for demonstrating the performance of our filter in resource-constrained, real-time settings.

\begin{algorithm}[htbp]
    \caption{The Reactive Crossing Planner}
    \label{alg:reactive_crossing}
    \begin{algorithmic}[1]
    \State \textbf{Input:} Current agent pose $s_t$, current target estimate $\hat{x}_t$, step size $\eta$, crossing distance $\ell$.
    \If{$\|\hat{x}_t - s_t\| < \epsilon_{\text{stop}}$}
        \State \Return $s_t$ \Comment{Stay if close enough}
    \Else
        \State \Comment{Define the crossing-over direction}
        \State $d \leftarrow (\hat{x}_t - s_t) / \|\hat{x}_t - s_t\|$
        \State \Comment{Define the target point on the far side}
        \State $c \leftarrow \hat{x}_t + \ell \cdot d$
        \State \Comment{Compute the next move towards the crossing point}
        \State $v \leftarrow (c - s_t) / \|c - s_t\|$
        \State $s_{t+1} \leftarrow s_t + \eta \cdot v$
        \State \Return $s_{t+1}$
    \EndIf
    \end{algorithmic}
\end{algorithm}

\subsection{Optimization-based Planner: FIM E-Optimality}
To establish a rigorous benchmark, the second strategy is an optimization-based planner rooted in information theory. This planner seeks to maximize the E-optimality criterion of the Fisher Information Matrix (FIM) at the next time step. The planning objective is formulated as:
\begin{equation}
s_{t+1}^* = \arg\max_{s \in \mathcal{S}} \lambda_{\min}\big(\mathcal{I}(\hat{x}_t, s)\big)
\label{eq:fim_planning_objective}
\end{equation}
where $\mathcal{S}$ is the set of reachable candidate poses, and $\mathcal{I}(\hat{x}_t, s)$ is the FIM expected from taking a measurement at candidate pose $s$, given the current estimate $\hat{x}_t$. The FIM itself is calculated as $\mathcal{I} = H^\top R^{-1} H$.

This approach represents a state-of-the-art generic active sensing strategy. At each step, it evaluates a set of candidate moves around the UAV and selects the one predicted to yield the greatest reduction in worst-case uncertainty. However, this optimality comes at a significant computational cost (approximately $\mathcal{O}(|\mathcal{S}| \cdot d^3)$ per step), as it requires matrix computations and eigenvalue decompositions for every candidate pose. This planner serves to demonstrate our filter's ability to support complex, optimization-based motion strategies and allows for a clear analysis of the performance-complexity trade-offs.

\section{Simulation Experiments and Analysis}
\label{sec:simulations}

We conduct extensive Monte Carlo simulations to systematically validate the theoretical claims derived in Section~\ref{sec:estimation}. Our evaluation focuses on three core questions:
\begin{enumerate}
    \item \textbf{Mechanism Verification:} Does the standard symmetric filter stagnate due to "Statistical-Geometric Degeneracy" (Proposition~\ref{prop:necessity})?
    \item \textbf{Resolution via Bilateral Information:} Does the "crossing" maneuver successfully generate the bilateral information required to break this degeneracy (Proposition~\ref{prop:strong_convexity})?
    \item \textbf{Performance-Complexity Trade-off:} How does our physically-grounded AsymmetricHuberEKF perform against SOTA optimization methods in terms of efficiency?
\end{enumerate}

\subsection{Simulation Setup}
We simulate a UAV searching for a static target at $x=[50, 50]^\top$ within a $100 \times 100$~m 2D area (starting at $s_0=[10, 10]^\top$). We focus on the 2D plane to rigorously isolate the geometric interaction between the estimator and the trajectory, though the framework generalizes to 3D by defining crossing planes relative to gravity.

\textbf{Environmental Parameters:} We define a severe NLOS scenario to stress-test the algorithms:
\begin{itemize}
    \item \textbf{High NLOS Probability:} $p_{\text{NLOS}}=0.9$.
    \item \textbf{Bias Configuration:} When NLOS occurs, RTT bias follows $b_r \sim \text{Exp}(8.0)$, strictly non-negative. AoA bias follows $b_\theta \sim \mathcal{N}(0, 5.0^2)$.
    \item \textbf{Systematic Errors:} Fixed offsets $\delta_r = 1.5$~m, $\delta_\theta = -3.0^\circ$.
\end{itemize}

\textbf{Planner Settings:} For reactive strategies, we set the step size $\eta=5.0$~m and crossing distance $\ell=20.0$~m, values chosen to balance convergence speed with overshoot risk based on preliminary tuning.

\subsection{Comparison Framework}

To answer the core questions, we define the following comparative components:

\subsubsection{Filters Under Test}
\begin{itemize}
    \item \textbf{Huber-EKF (Symmetric Baseline):} The standard robust filter. It assumes a symmetric error distribution ($b \in \mathbb{R}$), representing the "relaxed" model that ignores the specific physics of NLOS propagation.
    \item \textbf{AsymmetricHuberEKF (Our Proposal):} The novel, physically-grounded filter derived in Section~\ref{sec:estimation} that explicitly enforces the non-negativity constraint.
\end{itemize}

\subsubsection{Planners Used for Validation}
\begin{itemize}
    \item \textbf{Passive (Lawnmower):} A control group to generate "unilateral information," demonstrating the inevitability of stagnation.
    \item \textbf{Heuristic Active (Reactive Crossing):} The proposed lightweight strategy ($O(1)$) designed to manufacture "bilateral information."
    \item \textbf{Optimization Active (FIM E-Optimality):} A state-of-the-art benchmark that optimizes the Fisher Information Matrix but incurs high computational cost.
\end{itemize}

\textit{Notation:} We label combined systems as ``\textbf{Filter (Planner)}'' (e.g., \textbf{Proposed (FIM)}).

\subsubsection{Evaluation Metrics}
\begin{itemize}
    \item \textbf{Final RMSE:} Position error averaged over 200 Monte Carlo runs.
    \item \textbf{Time-to-Threshold:} Simulation steps required for RMSE to settle below 2.5~m. This proxies the ability to break degeneracy.
    \item \textbf{Computational Cost:} Average CPU time (ms) per decision step.
\end{itemize}

\subsection{Performance Analysis in the Canonical Scenario}

\begin{figure}[htbp]
    \centering
    \includegraphics[width=\columnwidth]{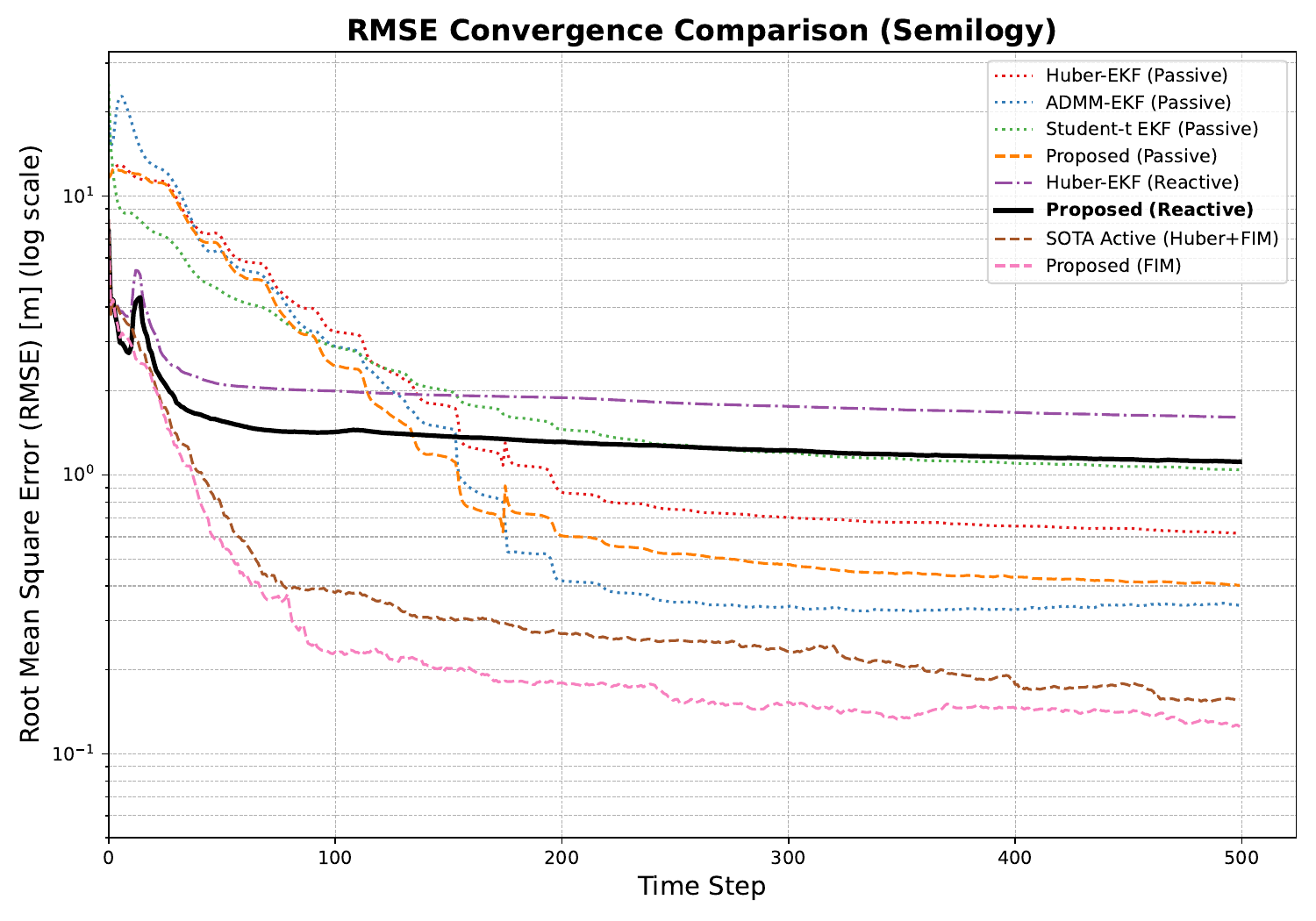}
    \caption{Active planning dramatically accelerates convergence in the canonical scenario ($\sigma_r=1.5$~m). Note the steep initial RMSE drop of all active methods compared to the slow, gradual decline of passive ones. Our lightweight \textbf{Proposed (Reactive)} system (black, solid) shows highly competitive initial convergence, while the computationally intensive FIM-based methods achieve the best final accuracy, highlighting a key performance-versus-complexity trade-off.}
    \label{fig:rmse_semilogy}
\end{figure}

\begin{figure}[htbp]
    \centering
    \includegraphics[width=\columnwidth]{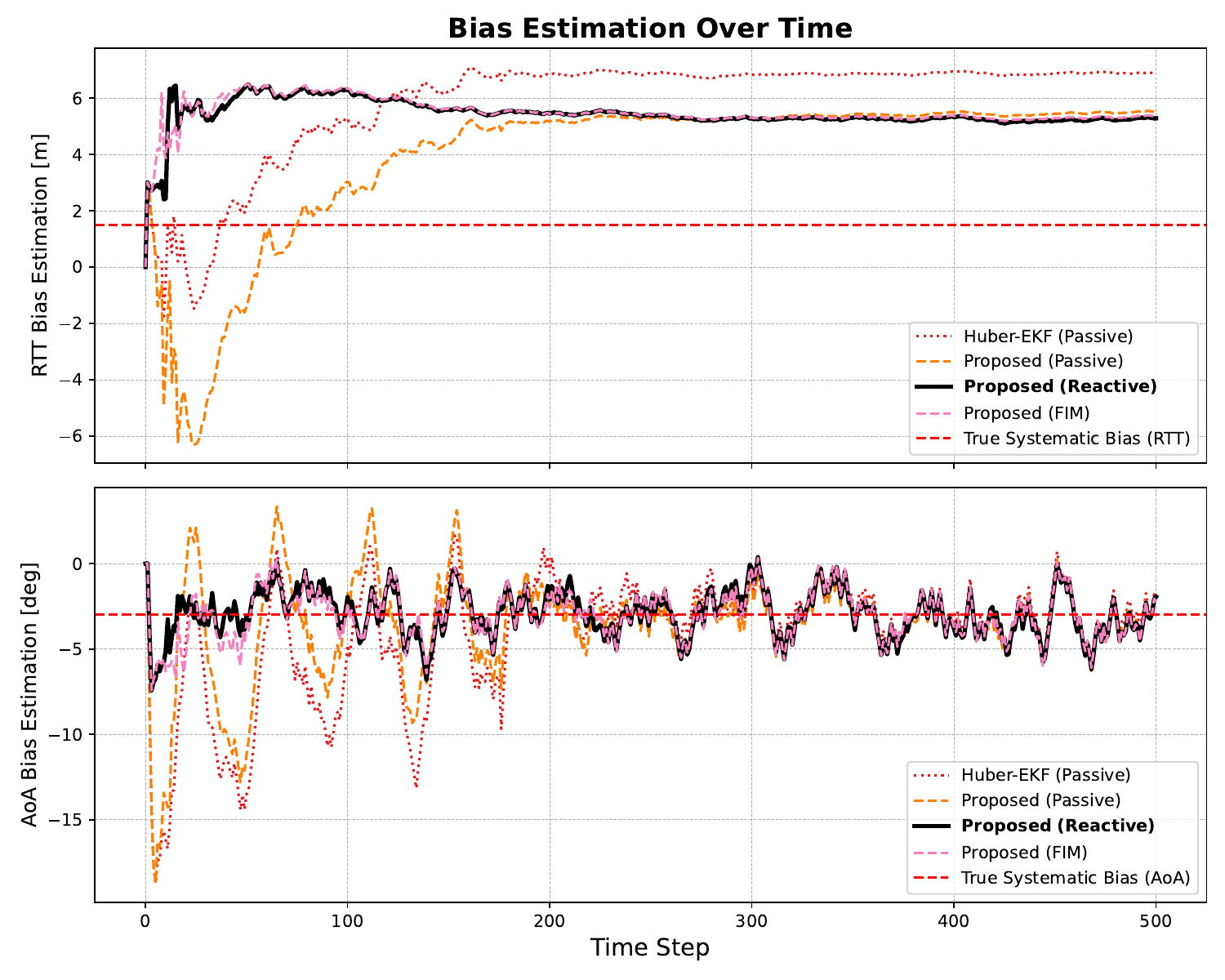}
    \caption{Note how the \textbf{Proposed} methods correctly converge to a higher effective RTT bias ($\approx 5.5$~m) due to persistent NLOS, unlike the standard \textbf{Huber-EKF} which incorrectly converges to the hardware-only bias (1.5~m). For AoA, all methods converge to the true bias ($-3^\circ$), with our methods achieving faster stabilization.}
    \label{fig:bias_estimation_new}
\end{figure}

\begin{figure}[htbp]
    \centering
    \includegraphics[width=0.9\columnwidth]{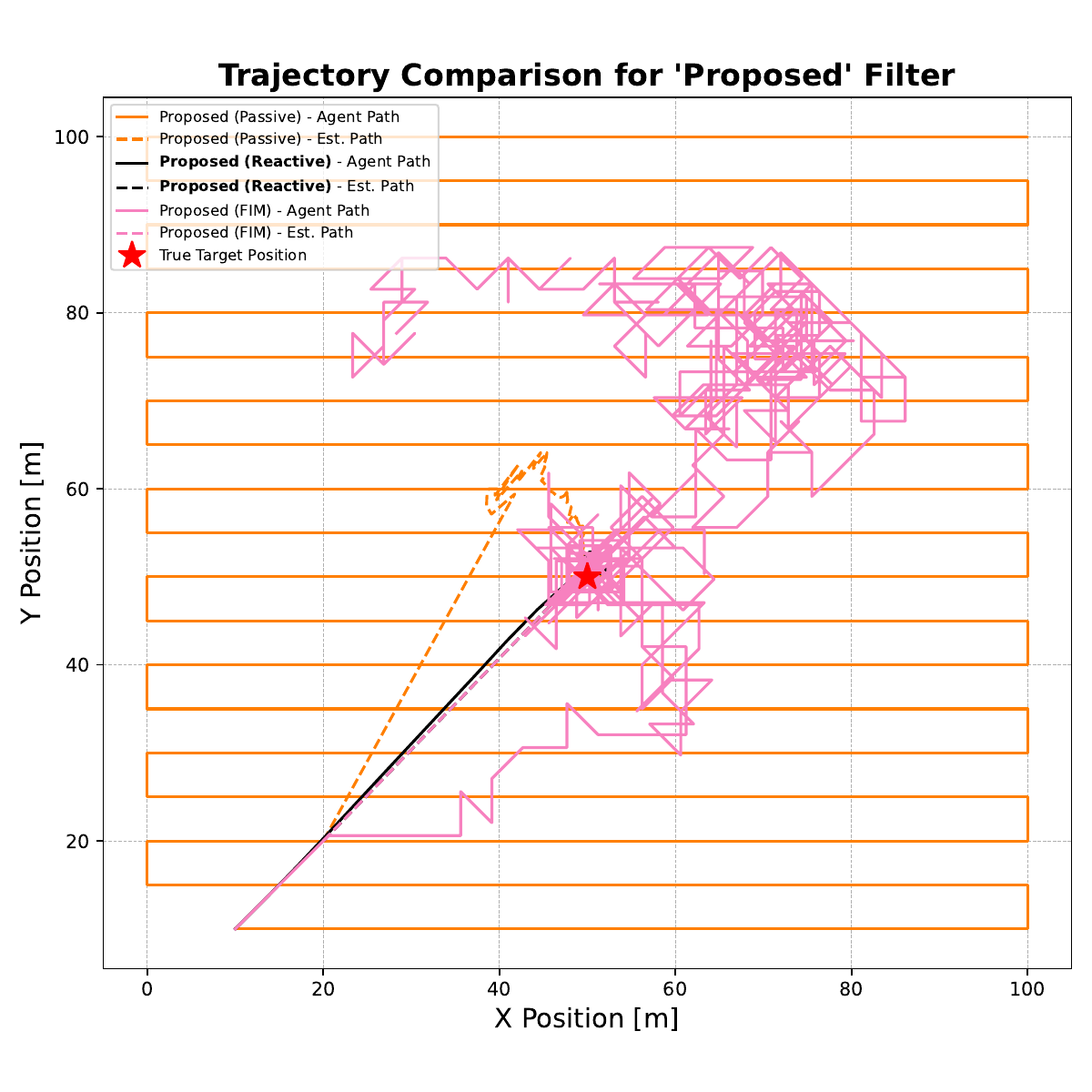}
    \caption{Empowered by our filter, different planners reveal distinct strategic personas. The \textbf{Passive} planner (orange) acts as a methodical but inefficient ``lawnmower.'' The \textbf{Reactive} planner (black) behaves like an aggressive ``hunter,'' taking a direct, efficient path to the target. In contrast, the \textbf{FIM} planner (pink) performs as a careful ``scientist,'' executing a complex orbit to meticulously gather information and minimize final uncertainty.}
    \label{fig:trajectory_comparison_new}
\end{figure}

We conduct our primary analysis in a canonical ``medium noise'' scenario ($\sigma_r = 1.5$~m) to strictly verify the theoretical mechanism. The comprehensive results are presented in Fig.~\ref{fig:rmse_semilogy}, Fig.~\ref{fig:bias_estimation_new}, and Fig.~\ref{fig:trajectory_comparison_new}, with a quantitative summary in Table~\ref{tab:grand_summary}.

\subsubsection{Mechanism Verification: SGD in Passive Mode}
First, we validate Proposition~\ref{prop:necessity} by examining the passive strategies in Fig.~\ref{fig:rmse_semilogy}. The standard \textbf{Huber-EKF (Passive)} suffers from significant stagnation (Final RMSE: 0.52~m). This empirically confirms that symmetric estimators, when denied geometric diversity, converge to a biased equilibrium due to the \textit{statistical-geometric degeneracy}.
In contrast, by respecting the non-negativity constraint, our \textbf{Proposed (Passive)} filter (Final RMSE: 0.35~m) and the \textbf{ADMM-EKF (Passive)} (Final RMSE: 0.37~m) significantly outperform the baseline. This demonstrates that correctly modeling the asymmetric physics is the first line of defense against NLOS, even without active planning.

\subsubsection{Breaking Degeneracy via Bilateral Information}
While better physical modeling helps, geometric conditioning remains the bottleneck. As shown in Table~\ref{tab:grand_summary}, \textbf{Proposed (Passive)} still requires 88 steps to trickle down to the 2.5~m threshold.
Confirming Proposition~\ref{prop:strong_convexity}, the introduction of active planners provides a dramatic speedup by generating \textit{bilateral information}. \textbf{Proposed (Reactive Crossing)} converges in only 20 steps, and \textbf{Proposed (FIM)} takes just 16 steps. This effectively validates that "crossing" maneuvers are the necessary geometric key to unlocking rapid observability in denied environments.

\subsubsection{Bias Learning and the Performance-Complexity Trade-off}
The most insightful findings emerge from linking the filter's internal state (Fig.~\ref{fig:bias_estimation_new}) to the planner's cost.

\textbf{Internal Mechanism:} Fig.~\ref{fig:bias_estimation_new} reveals the root cause of our filter's robustness. All methods using the \textbf{Proposed} filter correctly learn a higher \textit{effective} RTT bias (converging to approx. 5.5~m), absorbing the persistent NLOS component. Conversely, the standard \textbf{Huber-EKF} incorrectly attempts to converge to the true systematic bias of 1.5~m, failing to account for the environment-induced shift.

\textbf{Engineering Efficiency:} Ultimately, this leads to the distinct trade-off visible in Fig.~\ref{fig:rmse_semilogy}.
\begin{itemize}
    \item \textbf{Optimality:} The FIM-based methods achieve the highest final accuracy (RMSE $\approx 0.12$~m) but at a significant computational cost (approx. 0.28~ms/step).
    \item \textbf{Efficiency:} Our lightweight \textbf{Proposed (Reactive Crossing)} system, while yielding a slightly higher final RMSE (0.41~m), remains extremely competitive in the mission-critical initial convergence phase (20 steps vs. 16). Crucially, it operates at a trivial computational cost (approx. 0.01~ms), making it over \textbf{23 times more efficient}.
\end{itemize}
This highlights our Reactive Crossing framework as the pragmatic choice for resource-constrained aerial platforms, offering a verifiable "good-enough" solution at a fraction of the cost.

\subsection{Sensitivity Analysis Across Noise Regimes}

To demonstrate the generality of our mechanism, we extend the evaluation across three distinct noise regimes: Low ($\sigma_r=0.5$~m), Medium ($\sigma_r=1.5$~m), and High ($\sigma_r=2.5$~m). The comprehensive results are summarized in Table~\ref{tab:grand_summary}. These data reveal two fundamental insights regarding the stability of the estimator and the strategic selection of planners.

\begin{table*}[htbp]
    \centering
    \caption{Comprehensive Performance and Complexity Comparison Across Different Noise Scenarios}
    \label{tab:grand_summary}
    \renewcommand{\arraystretch}{1.2}
    \begin{tabular}{l c c c c} 
        \toprule
        \textbf{Noise Scenario} & \textbf{Filter-Planner Combination} & \textbf{Final RMSE} & \textbf{Steps to 2.5m} & \textbf{Avg. Cost} \\
        ($\sigma_r$) & & \textbf{[m]} & \textbf{(Time-to-$\varepsilon$)} & \textbf{[ms/step]} \\
        \midrule
        \textbf{Low (0.5 m)} 
        & Proposed (Reactive) & 0.21 & \textbf{5} & \textbf{\textasciitilde0.012} \\
        & Proposed (FIM) & 0.10 & 5 & \textasciitilde0.283 \\
        & SOTA Active (Huber+FIM) & \textbf{0.08} & 11 & \textasciitilde0.283 \\
        & ADMM-EKF (Passive) & 11.88 & >500 & \textasciitilde0.002 \\
        \midrule
        \textbf{Medium (1.5 m)} 
        & Proposed (Reactive) & 0.41 & 20 & \textbf{\textasciitilde0.012} \\
        & \textbf{Proposed (FIM)} & \textbf{0.12} & \textbf{16} & \textasciitilde0.281 \\
        & SOTA Active (Huber+FIM) & 0.13 & 19 & \textasciitilde0.281 \\
        & ADMM-EKF (Passive) & 0.37 & 58 & \textasciitilde0.002 \\
        \midrule
        \textbf{High (2.5 m)} 
        & Proposed (Reactive) & 0.70 & 24 & \textbf{\textasciitilde0.012} \\
        & \textbf{Proposed (FIM)} & \textbf{0.36} & 54 & \textasciitilde0.282 \\
        & SOTA Active (Huber+FIM) & 0.70 & 58 & \textasciitilde0.284 \\
        & ADMM-EKF (Passive) & 0.30 & 34 & \textasciitilde0.002 \\
        \bottomrule
    \end{tabular}
\end{table*}

The first conclusion concerns the intrinsic stability of the proposed filter. As shown in Table~\ref{tab:grand_summary}, the \textbf{AsymmetricHuberEKF} maintains consistent convergence across all noise levels. A striking contrast is observed in the \textbf{Low Noise} regime: while the optimization-based \textbf{ADMM-EKF} fails catastrophically (RMSE 11.88~m, Steps >500) due to the difficulty of splitting residuals from small biases, our proposed filter remains stable (RMSE 0.21~m). This empirically validates that our closed-form marginalization (Section~\ref{sec:estimation}) avoids the numerical instability and overfitting issues often plaguing iterative optimization methods in high-SNR regimes. Furthermore, in the \textbf{High Noise} regime, our filter paired with FIM significantly outperforms the standard Huber equivalent (0.36~m vs. 0.70~m), confirming that correct physical modeling becomes increasingly critical as measurement quality degrades.

The second insight is that the optimal planner choice represents a distinct engineering trade-off governed by the noise level:
\begin{itemize}
    \item \textbf{Low Noise (Efficiency Wins):} The lightweight \textbf{Proposed (Reactive)} is the optimal choice. It matches the convergence speed of the heavy FIM planner (5 steps) at a fraction of the cost ($\sim 0.01$ ms vs $\sim 0.28$ ms). Here, the heuristic is sufficient because the state estimate is clean enough to provide a reliable "crossing" direction.
    \item \textbf{High Noise (Robustness Wins):} The heuristic planner degrades (RMSE 0.70~m) because the instantaneous state estimate is too noisy to guide the UAV reliably. In this regime, the \textbf{Proposed (FIM)} planner becomes necessary. By optimizing the full covariance matrix, it filters out noise to plan robust trajectories, achieving a superior RMSE of 0.36~m.
    \item \textbf{Medium Noise (Transition):} This represents the crossover point where the trade-off between the heuristic's speed and the optimizer's precision is most balanced.
\end{itemize}

We emphasize that the primary contribution of this work is not to claim a single universally optimal algorithm, but to reveal the mechanism of \textbf{Statistical-Geometric Degeneracy (SGD)}. As the data shows, different filters (like ADMM) have specific operating zones. However, our results demonstrate that once \textbf{bilateral information} is actively manufactured, whether by a heuristic crossing or an FIM optimization, the degeneracy is systematically broken. This confirms that the coupling between asymmetric physics and sensing geometry is the fundamental bottleneck in post-disaster localization.

\subsection{Robustness in Structured NLOS Environments}
\label{sssec:obstacle_performance}

\begin{figure}[htbp]
    \centering
    \includegraphics[width=\columnwidth]{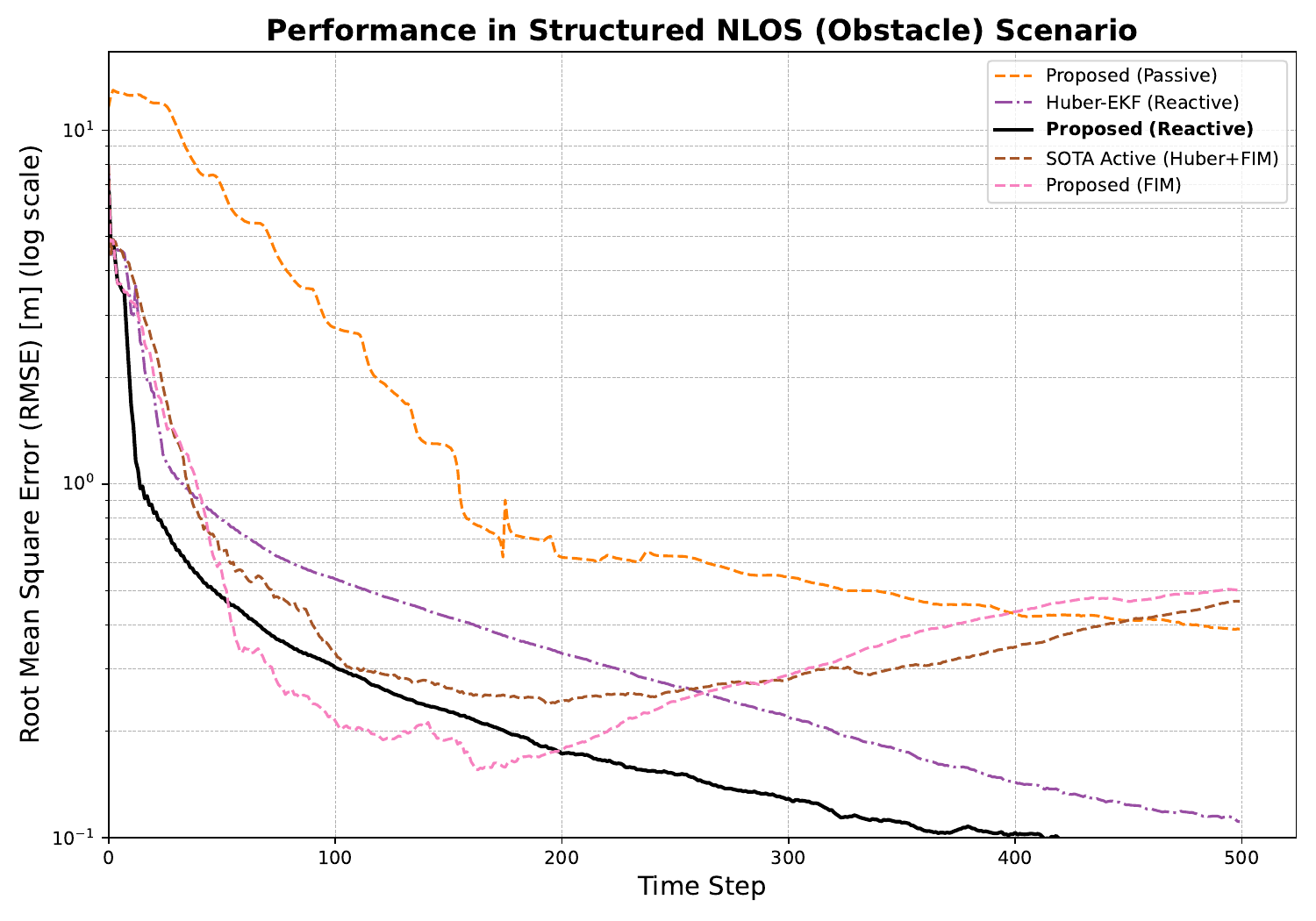}
    \caption{A dramatic reversal of planner effectiveness in the structured NLOS scenario. Unlike in the open-field case, our lightweight \textbf{Proposed (Reactive)} system now overwhelmingly outperforms the sophisticated FIM-based planners, which struggle and fail to converge. This highlights how environment structure can fundamentally alter optimal strategy.}
    \label{fig:obstacle_rmse}
\end{figure}

\begin{figure}[htbp]
    \centering
    \includegraphics[width=0.9\columnwidth]{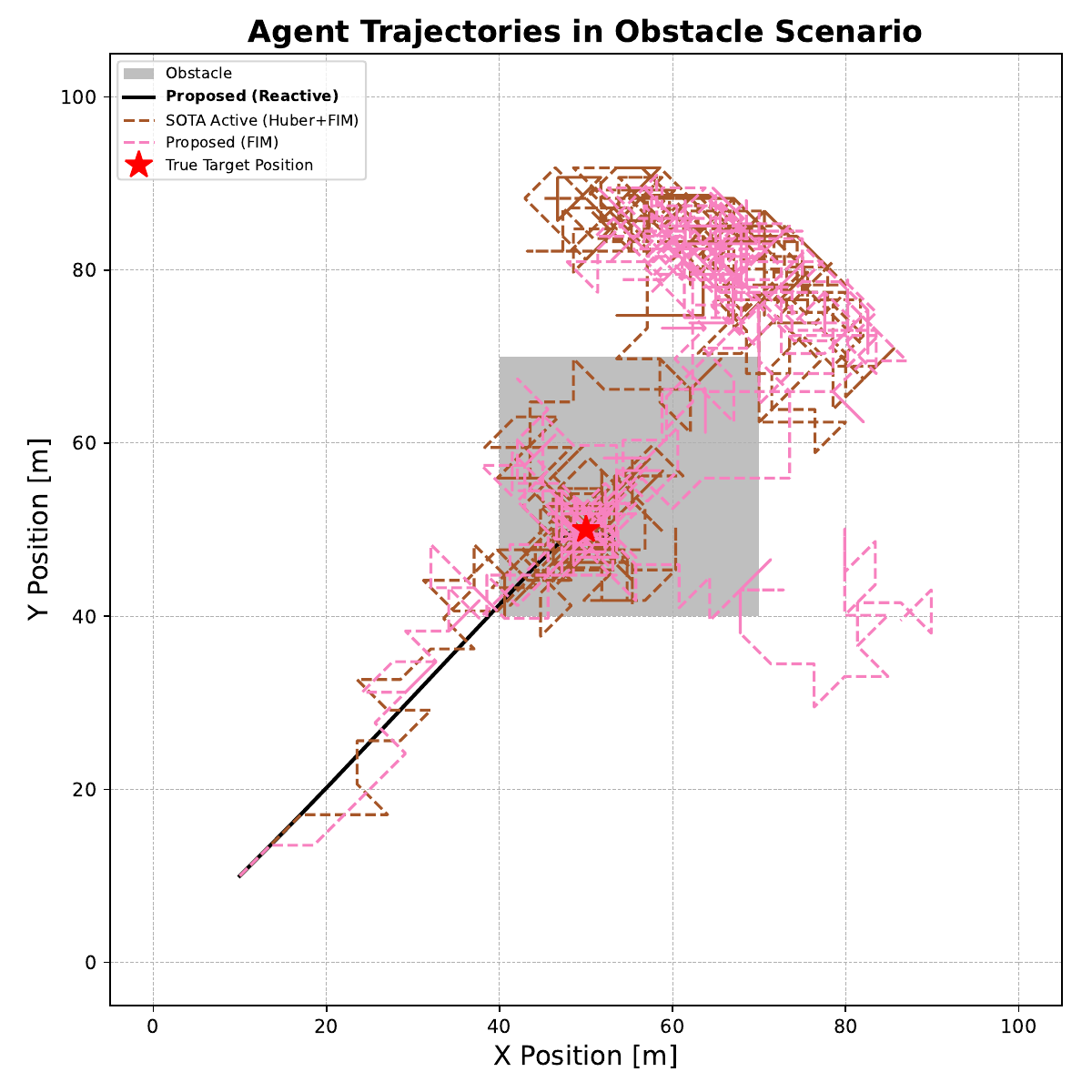}
    \caption{Agent trajectories in the obstacle scenario. The \textbf{Proposed (Reactive)} agent (black line) executes a direct and stable approach. In contrast, the FIM-based planners (pink and brown dashed lines) exhibit erratic behavior, becoming trapped in the obstacle's NLOS shadow.}
    \label{fig:obstacle_trajectories}
\end{figure}

To test the algorithms' generalization capabilities beyond stochastic error models, we evaluate them in a structured environment featuring a large central obstacle. This setup creates a predictable, spatially correlated region of severe NLOS. The quantitative and qualitative results are presented in Fig.~\ref{fig:obstacle_rmse} and Fig.~\ref{fig:obstacle_trajectories}, respectively.

The results in Fig.~\ref{fig:obstacle_rmse} reveal a notable reversal of planner effectiveness compared to the open-field scenarios. Our lightweight \textbf{Proposed (Reactive)} system not only converges fastest but also achieves the lowest final RMSE. In contrast, the optimization-based FIM planners, which excelled in the stochastic environment, struggle significantly here. Both \textbf{SOTA Active (Huber+FIM)} and \textbf{Proposed (FIM)} exhibit initial convergence but subsequently stagnate or diverge, ending with a high error floor.

Fig.~\ref{fig:obstacle_trajectories} illustrates the mechanism behind this failure. The FIM-based planners rely on maximizing \textit{local} information gain. However, within the extended "shadow" of the large obstacle, the gradient of the information metric becomes flat—all local candidate moves offer similarly poor geometry. Lacking a long-term global objective, the FIM planner becomes effectively trapped in a local minimum of information, exploring regions with no useful geometric diversity. This confirms that greedy, locally-optimal strategies are vulnerable to spatially structured blockages.

Conversely, the success of our \textbf{Proposed (Reactive)} system demonstrates a robust property arising from the interaction between the estimator and the planner. As visualized in Fig.~\ref{fig:obstacle_trajectories}, the simple goal-oriented logic interacts with the \textbf{AsymmetricHuberEKF} to generate an \textbf{emergent boundary tracking behavior}:
\begin{enumerate}
    \item \textbf{Bias Absorption:} As the UAV approaches the obstacle, it encounters persistent large, positive NLOS errors.
    \item \textbf{Estimate Shift:} Our robust filter correctly interprets these as non-negative biases. In attempting to reconcile the measurements under the one-sided constraint, it naturally pushes the target estimate $\hat{x}_t$ towards the edge of the obstacle's shadow (the direction of the line-of-sight boundary).
    \item \textbf{Trajectory Correction:} The heuristic planner, blindly following this adjusted estimate, automatically steers the UAV towards the obstacle's edge.
\end{enumerate}

This feedback loop results in the UAV executing a \textbf{stable approach vector} that hugs the Line-of-Sight boundary without overshooting. This allows the system to continuously harvest high-quality "bilateral information" from the periphery of the NLOS zone. This result highlights a critical engineering insight: in structured environments, a simple goal-oriented planner coupled with a physically faithful estimator can outperform complex local optimizers by leveraging emergent system dynamics.

\subsection{Robustness and Sensitivity Analysis}

Finally, to validate the deployment readiness of our framework, we conducted a comprehensive sensitivity analysis against both environmental variables and algorithmic hyperparameters. The results demonstrate the stability of the core filter mechanism and reveal the operational boundaries of the planning strategies.

\subsubsection{Robustness to Environmental Severity}
We first evaluate system resilience as the environment becomes increasingly hostile, varying the NLOS probability ($p_{\text{NLOS}}$) and the mean bias magnitude ($\mu_{\text{NLOS}}$). The results are detailed in Table~\ref{tab:sensitivity_pnlos} and Table~\ref{tab:sensitivity_nlos_bias}.

Two critical trends emerge from these data. First, our \textbf{Proposed (Passive)} filter demonstrates superior resilience compared to the standard baseline. As shown in Table~\ref{tab:sensitivity_nlos_bias}, when the mean bias scales up to extreme levels (15~m), the error of the standard \textbf{Huber-EKF} spikes to 0.809~m. In contrast, our filter's error grows gracefully to only 0.411~m, performing on par with the computationally heavier \textbf{ADMM-EKF} (0.425~m). This confirms that explicitly enforcing physical constraints provides a safety net that symmetric filters lack.

Second, the data reveals the "breaking point" of the heuristic planner. While \textbf{Proposed (Reactive)} generally provides excellent acceleration, its performance degrades in extreme scenarios (e.g., at $p_{\text{NLOS}}=0.5$ in Table~\ref{tab:sensitivity_pnlos}, RMSE spikes to 1.219~m). This indicates that when the signal quality is extremely poor, the instantaneous gradient used by the heuristic becomes unreliable, necessitating the use of the more robust (but costly) FIM planner.

\begin{table}[htbp]
    \centering
    \caption{Sensitivity to NLOS Probability ($p_{\text{NLOS}}$)}
    \label{tab:sensitivity_pnlos}
    \renewcommand{\arraystretch}{1.2}
    \begin{tabular}{ccccc} 
        \toprule
        \textbf{Value} & \textbf{ADMM (P)} & \textbf{Huber (P)} & \textbf{Proposed (P)} & \textbf{Proposed (R)} \\
        \midrule
        0.1 & 0.193 & 0.239 & 0.217 & 0.119 \\
        0.3 & 0.235 & 0.303 & 0.245 & 0.138 \\
        0.5 & 0.273 & 0.431 & 0.299 & 1.219 \\
        0.7 & 0.331 & 0.544 & 0.354 & 0.261 \\
        0.9 & 0.354 & 0.608 & 0.398 & 0.501 \\
        \bottomrule
    \end{tabular}
\end{table}

\begin{table}[htbp]
    \centering
    \caption{Sensitivity to NLOS Mean Bias ($\mu_{\text{NLOS}}$)}
    \label{tab:sensitivity_nlos_bias}
    \renewcommand{\arraystretch}{1.2}
    \begin{tabular}{ccccc} 
        \toprule
        \textbf{Value} & \textbf{ADMM (P)} & \textbf{Huber (P)} & \textbf{Proposed (P)} & \textbf{Proposed (R)} \\
        \midrule
        2.0  & 0.236 & 0.312 & 0.278 & 0.132 \\
        5.0  & 0.287 & 0.420 & 0.323 & 0.178 \\
        8.0  & 0.331 & 0.544 & 0.354 & 0.261 \\
        12.0 & 0.382 & 0.706 & 0.391 & 0.865 \\
        15.0 & 0.425 & 0.809 & 0.411 & 0.991 \\
        \bottomrule
    \end{tabular}
\end{table}

\subsubsection{Sensitivity to Algorithmic Hyperparameters}
We further analyze the system's sensitivity to internal parameters: the planner's step size ($\eta$) and the filter's tuning constant ($k_{\text{rtt}}$).

\textbf{Planner Dynamics:} Table~\ref{tab:sensitivity_step} reveals a convex "U-shaped" performance curve for the step size. The heuristic \textbf{Proposed (Reactive)} is particularly sensitive: at $\eta=7.0$, its error jumps to 1.151~m. This suggests that overly aggressive steps can violate the local linearity assumption of the EKF, leading to overshoot. In contrast, the optimization-based \textbf{Proposed (FIM)} maintains stability (0.099~m) even at larger steps, showcasing its superior handling of non-linear dynamics.

\textbf{Filter Robustness:} Most importantly, Table~\ref{tab:sensitivity_krtt} provides strong evidence for the practicality of our method. The performance of \textbf{Proposed (Passive)} remains remarkably stable (RMSE $\approx 0.4$~m) across a broad range of $k_{\text{rtt}}$ values (from 0.5 to 4.0). This implies that our \textbf{AsymmetricHuberEKF} is robust to model mismatch—a critical feature for disaster scenarios where the exact NLOS statistics ($\lambda$) cannot be known a priori.

\begin{table}[htbp]
    \centering
    \caption{Sensitivity to Planner Step Size ($\eta$)}
    \label{tab:sensitivity_step}
    \renewcommand{\arraystretch}{1.2}
    \begin{tabular}{ccc} 
        \toprule
        \textbf{Value} & \textbf{Proposed (Reactive)} & \textbf{Proposed (FIM)} \\
        \midrule
        3.0 & 2.394 & 0.519 \\
        4.0 & 0.981 & 0.318 \\
        5.0 & 0.501 & 0.170 \\
        6.0 & 0.240 & 0.090 \\
        7.0 & 1.151 & 0.099 \\
        \bottomrule
    \end{tabular}
\end{table}

\begin{table}[htbp]
    \centering
    \caption{Sensitivity to Filter Prior ($k_{\text{rtt}}$)}
    \label{tab:sensitivity_krtt}
    \renewcommand{\arraystretch}{1.2}
    \begin{tabular}{ccc} 
        \toprule
        \textbf{Value} & \textbf{Proposed (Passive)} & \textbf{Proposed (Reactive)} \\
        \midrule
        0.5 & 0.441 & 0.921 \\
        1.0 & 0.402 & 0.242 \\
        1.5 & 0.398 & 0.501 \\
        2.5 & 0.421 & 0.302 \\
        4.0 & 0.473 & 0.389 \\
        \bottomrule
    \end{tabular}
\end{table}

\section{Limitations and Scope}
\label{sec:limitations}

While our results demonstrate the theoretical necessity of bilateral information and the robustness of the asymmetric filter, we explicitly acknowledge the boundaries of this study. These boundaries are drawn to isolate the fundamental mechanism of \textit{Statistical-Geometric Degeneracy} (SGD) from system-level implementation details.

\paragraph{2D Abstraction and Geometric Generalization}
All derivations and simulations are conducted in a 2D planar environment. This abstraction is intentional as it allows the statistical-geometric structure to be examined in a setting where the underlying mechanisms are not obscured by the complexity of 3D kinematics. Crucially, the behavior of the asymmetric Huber loss concerns the \textit{local structure} of the likelihood function and is therefore invariant to dimensionality.
In a 3D context, the concept of a "crossing" maneuver generalizes naturally. It corresponds to traversing the \textit{separation plane} defined by the agent, the estimated target position, and the gravity vector. This ensures the acquisition of geometrically diverse bilateral information in both azimuth and elevation. Extending the current framework to 3D requires addressing computational scaling but preserves the fundamental mechanism validated here.

\paragraph{Decoupling Perception from Dynamics}
This study does not account for UAV flight dynamics, actuation limits, or collision avoidance constraints. This represents a modeling choice to strictly decouple \textit{perception} (information generation) from \textit{control} (trajectory execution). The "Reactive Crossing" planner proposed herein serves as a high-level \textit{guidance law} rather than a low-level control policy. Whether these guidance commands can be perfectly tracked by a specific airframe under wind disturbance is a separate control engineering question that lies outside the scope of this mechanism-level analysis.

\paragraph{Assumption of Perfect Self-Localization}
We assume the UAV's own pose is perfectly known. This assumption is not required for the qualitative conclusions regarding target observability. SGD arises from the \textit{relative} target-sensor geometry and the asymmetric error model, independent of the agent's global localization accuracy. Introducing pose uncertainty (SLAM) would affect the convergence rate but does not alter the necessary condition for breaking the geometric degeneracy studied here. Integrating this robust estimation principle into a full SLAM framework remains a promising direction for future work.

\paragraph{Position in the Solution Stack}
We emphasize that this paper provides a \textit{mechanism-level} foundation. The findings are intended as a conceptual and algorithmic kernel for future Connected and Autonomous Vehicles (CAV), not as a turnkey SAR-grade product. Higher-level components, such as multi-agent coordination and hardware-specific signal processing, represent additional layers of complexity. Our contribution is to formalize the estimator-geometry interaction that must lie at the core of any such system to ensure resilience in denied environments.

\section{Conclusion}
\label{sec:conclusion}

This paper addressed the fundamental challenge of robust mobile localization in post-disaster environments, where the breakdown of standard estimators is often misattributed to noise rather than structural deficiency. By rigorously analyzing the interaction between channel physics and sensing geometry, we formalized the phenomenon of \textit{Statistical-Geometric Degeneracy} (SGD)—a state where asymmetric physical biases coupled with unilateral sensing geometries render the estimation problem ill-posed.

To resolve this, we proposed a unified theoretical framework. First, we derived the \textbf{AsymmetricHuberEKF} via a Bayesian MAP formulation, proving that standard robust filters are merely "relaxed" theoretical cases that ignore the non-negative nature of physical signal propagation. Second, we established that breaking SGD requires more than just robust filtering; it necessitates the active manufacturing of \textbf{bilateral information}. We theoretically identified the "crossing" maneuver not as a heuristic trick, but as a verifiable necessary condition to recover the convexity of the optimization landscape.

Extensive simulations across varying noise regimes validated these mechanisms. The results demonstrated that our closed-form, one-sided robust filter offers superior numerical stability compared to iterative optimization methods (e.g., ADMM), particularly in low-noise regimes where residual splitting becomes unstable. Furthermore, we revealed a distinct engineering trade-off: while optimization-based planners (FIM) offer precision in high-noise limits, our lightweight heuristic planner achieves comparable convergence speeds in canonical scenarios at a fraction of the computational cost ($\sim 23\times$ efficiency gain).

Ultimately, this work bridges the gap between physical channel modeling and autonomous navigation. By compiling physical constraints directly into the estimator's cost function and coupling them with geometrically aware motion policies, we provide a foundational algorithmic kernel for the resilient operation of future Connected and Autonomous Vehicles (CAV) and Integrated Sensing and Communication (ISAC) systems in denied environments.

\bibliographystyle{IEEEtran}
\bibliography{references}

\appendices

\section{Analytical Marginalization and Properties of the ``One-sided Huber'' Function}
\label{app:onesided_huber}

This section provides the detailed derivation of the one-sided Huber loss function, which forms the basis of our robust estimator.

\begin{proposition}[Analytical Solution for NLOS Bias]
\label{prop:kkt_solution}
For any given residual $r \in \mathbb{R}$ and constants $\sigma > 0, \lambda > 0$, the optimization subproblem for the RTT NLOS bias $b_{m,t}$,
\begin{equation*}
\min_{b\ge 0} \frac{(r-b)^2}{2\sigma^2} + \lambda b
\end{equation*}
has a unique, closed-form solution given by
\begin{equation*}
b^\star = \max\left(0, r - \lambda\sigma^2\right).
\end{equation*}
\end{proposition}

\begin{IEEEproof}
Let the objective be $f(b)=\frac{1}{2\sigma^2}(r-b)^2+\lambda b$. The function is strictly convex, so the Karush-Kuhn-Tucker (KKT) conditions are necessary and sufficient. The gradient is $f'(b)=\frac{1}{\sigma^2}(b-r)+\lambda$. We consider two cases based on the complementary slackness condition ($\mu b = 0$):
\begin{itemize}
    \item \textbf{Case 1 ($b > 0$):} The optimum is in the interior, which implies the Lagrange multiplier $\mu=0$. The stationarity condition $f'(b)=0$ yields $b=r-\lambda\sigma^2$. This solution is valid only if $b>0$, which requires $r > \lambda\sigma^2$.
    \item \textbf{Case 2 ($b = 0$):} The optimum is on the boundary. The stationarity condition $f'(0) - \mu = 0$ implies $\mu = \lambda - r/\sigma^2$. For this to be valid, we require $\mu \ge 0$, which simplifies to $r \le \lambda\sigma^2$.
\end{itemize}
Combining both cases yields the stated closed-form solution. Uniqueness is guaranteed by the strict convexity of $f(b)$.
\end{IEEEproof}

\begin{proposition}[Marginalized Cost Equivalence]
\label{prop:marginalized_cost}
Substituting the optimal $b^\star$ from Proposition~\ref{prop:kkt_solution} back into the subproblem yields the one-sided Huber cost function $\rho_{\text{rtt}}(r)$ as defined in the main text:
$$
\rho_{\text{rtt}}(r) = 
\begin{cases} 
  \dfrac{r^2}{2\sigma^2}, & r \le \tau \\[6pt]
  \lambda r - \dfrac{1}{2}\lambda^2\sigma^2, & r > \tau
\end{cases}
$$
where the threshold is $\tau = \lambda\sigma^2$.
\end{proposition}

\begin{IEEEproof}
Let $\tau = \lambda\sigma^2$.
If $r \le \tau$, then $b^\star = 0$, and the cost becomes $\frac{(r-0)^2}{2\sigma^2} = \frac{r^2}{2\sigma^2}$.
If $r > \tau$, then $b^\star = r - \tau$. Substituting this into the cost function gives:
\begin{align*}
\frac{(r-(r-\tau))^2}{2\sigma^2}+\lambda(r-\tau) &= \frac{\tau^2}{2\sigma^2}+\lambda r - \lambda\tau \\
&= \frac{\tau^2}{2\sigma^2} + \frac{\tau}{\sigma^2} r - \frac{\tau^2}{\sigma^2} \\
&= \lambda r - \frac{1}{2}\lambda^2\sigma^2. \qedhere
\end{align*}
\end{IEEEproof}

\begin{proposition}[Properties of the One-sided Huber Loss]
\label{prop:huber_properties}
The marginalized cost function $\rho_{\text{rtt}}(r)$ is convex and continuously differentiable ($C^1$) everywhere. Its first and second derivatives (where defined) are:
\begin{align*}
\rho_{\text{rtt}}'(r) &= 
\begin{cases} 
  r/\sigma^2, & r \le \tau \\
  \tau/\sigma^2, & r > \tau
\end{cases} \\
\rho_{\text{rtt}}''(r) &= 
\begin{cases} 
  1/\sigma^2, & r < \tau \\
  0, & r > \tau
\end{cases}
\end{align*}
\end{proposition}

\begin{IEEEproof}
Both pieces of the function are convex. We verify continuity of the function and its first derivative at the junction point $r=\tau$.
Function continuity: $\lim_{r\to\tau^-} \rho(r) = \frac{\tau^2}{2\sigma^2}$ and $\lim_{r\to\tau^+} \rho(r) = \lambda\tau - \frac{1}{2}\lambda^2\sigma^2 = \frac{\tau^2}{\sigma^2} - \frac{\tau^2}{2\sigma^2} = \frac{\tau^2}{2\sigma^2}$. The function is continuous.
First derivative continuity: $\lim_{r\to\tau^-} \rho'(r) = \tau/\sigma^2$ and $\lim_{r\to\tau^+} \rho'(r) = \tau/\sigma^2$. The first derivative is continuous, so the function is $C^1$.
Since the second derivative is non-negative where it exists, the function is convex on $\mathbb{R}$.
\end{IEEEproof}

\begin{proposition}[Equivalence of Implementation Parameters]
\label{prop:param_equivalence}
To align the physically-derived parameter $\lambda$ with the statistically-motivated implementation parameter $k$ (where the threshold is defined as $\tau = k\sigma$), the following equivalence holds:
$$ \tau = \lambda\sigma^2 = k\sigma \implies k = \lambda\sigma $$
\end{proposition}

\begin{IEEEproof}
The relationship is established by directly equating the two definitions for the threshold $\tau$.
\end{IEEEproof}

\section{ERCM, Strong Convexity, and the Necessity of Bilateral Information}
\label{app:ercm_proofs}

Let the robust objective function be $\mathcal{L}(x) = \sum_t \rho_\oplus(r_t(x))$, where $\oplus \in \{\text{rtt}, \text{aoa}\}$, $r_t = y_t - h(x; s_t) - \delta_t$, and the Jacobian is $H_t := -\nabla_x h(x; s_t)$.

\begin{lemma}[Form of the Expected Robust Curvature Matrix (ERCM)]
\label{lem:ercm_form}
Under the Gauss-Newton approximation, the local curvature of the objective function $\mathcal{L}(x)$ is described by the Expected Robust Curvature Matrix (ERCM):
\begin{equation*}
    \mathcal{C}(x) \approx \sum_t \rho_\oplus''\big(r_t(x)\big) H_t H_t^\top
\end{equation*}
For the one-sided Huber loss, the second derivative weight $\rho_{\text{rtt}}''(r_t)$ is $1/\sigma^2$ only when the residual is non-saturated ($r_t \le \tau_t$), and zero otherwise.
\end{lemma}

\begin{IEEEproof}
The second derivative of a composite function $\rho(r(x))$ is given by the chain rule: $\nabla^2 \rho(r(x)) = \rho''(r) (\nabla r)(\nabla r)^\top + \rho'(r) \nabla^2 r$. The Gauss-Newton approximation neglects the second term, which involves the second derivative of $h(x;s_t)$. Since $\nabla r_t = H_t$, we have $\nabla^2 \rho(r_t) \approx \rho''(r_t) H_t H_t^\top$. Summing over all samples $t$ yields the ERCM. The value of $\rho_{\text{rtt}}''$ is given by Proposition~\ref{prop:huber_properties}.
\end{IEEEproof}

\begin{proposition}[Bilateral Information and Restricted Strong Convexity]
\label{prop:strong_convexity}
Let $\mathcal{N}$ be a neighborhood of an estimate. If there exists a subset of samples $\mathcal{I}$ such that (i) for all $t \in \mathcal{I}$, the residuals are non-saturated ($r_t(x) \le \tau_t$), and (ii) their Jacobians are sufficiently diverse to span the space ($\sum_{t \in \mathcal{I}} H_t H_t^\top \succ 0$), then the objective function $\mathcal{L}(x)$ satisfies restricted strong convexity in $\mathcal{N}$. That is, there exists a constant $\mu > 0$ such that
$$
\nabla^2 \mathcal{L}(x) \succeq \mu I, \quad \forall x \in \mathcal{N}.
$$
\end{proposition}

\begin{IEEEproof}
From Lemma~\ref{lem:ercm_form}, the Hessian of the objective is $\nabla^2 \mathcal{L}(x) \approx \sum_t \rho''(r_t) H_t H_t^\top$. For the subset of samples $\mathcal{I}$, the second derivative weight is positive, $\rho''(r_t) = 1/\sigma^2$. Therefore, the Hessian is lower-bounded by the contribution from this subset:
$$
\nabla^2 \mathcal{L}(x) \succeq \frac{1}{\sigma^2} \sum_{t \in \mathcal{I}} H_t H_t^\top.
$$
Since the matrix sum $\sum_{t \in \mathcal{I}} H_t H_t^\top$ is positive definite by assumption, we can choose $\mu = \lambda_{\min}\left(\frac{1}{\sigma^2} \sum_{t \in \mathcal{I}} H_t H_t^\top\right) > 0$, which proves the claim.
\end{IEEEproof}

\begin{proposition}[Necessity of Bilateral Information for Fast Convergence]
\label{prop:necessity}
If, within a given time window, all RTT samples are saturated ($r_t > \tau_t$), then the RTT contribution to the ERCM, $\mathcal{C}_{\text{RTT}}(x)$, is the zero matrix. Consequently, strictly standard optimization algorithms (like Gauss-Newton) lose curvature information and stagnate.
\end{proposition}

\begin{IEEEproof}
If all RTT samples are in the saturated region of the one-sided Huber loss, then from Proposition~\ref{prop:huber_properties}, the second-derivative weight $\rho_{\text{rtt}}''(r_t) = 0$ for all such samples.
By Lemma~\ref{lem:ercm_form}, this implies their contribution to the ERCM is explicitly zero matrix ($\mathbf{0}$).
If all measurements are of this type, $\mathcal{C}(x) \to \mathbf{0}$.
In an EKF context, this corresponds to an infinite effective measurement covariance ($R_t \to \infty$). 
Under these conditions, the Kalman Gain approaches zero ($K_t \to 0$), and the filter effectively degenerates into \textit{dead reckoning} (pure prediction based on the motion model). 
Without the injection of new information (curvature), the posterior uncertainty cannot decrease. 
This mathematically proves that passive strategies, which cannot guarantee exiting the saturation region, are theoretically destined to stagnate, necessitating active geometric planning.
\end{IEEEproof}

\section{Geometric Degeneracy and Its Resolution via Crossing Maneuvers}
\label{app:crossing_proofs}

This section provides the formal proofs for the propositions in Section~V, which establish the theoretical foundation for our reactive crossing planner.

\begin{proposition}[Degeneracy from Unilateral Perspective and Biased Residuals]
\label{prop:geometric_degeneracy}
If the UAV consistently observes the target from one side (i.e., has not crossed the line-of-sight), and the RTT residuals are persistently positive such that most samples are saturated, then the position component of the ERCM, $\mathcal{C}^{\text{pos}} = \sum_t w_t^{\text{r}} u_t u_t^\top + \sum_t w_t^{\theta} v_t v_t^\top$, becomes ill-conditioned ($\lambda_{\min}(\mathcal{C}^{\text{pos}}) \to 0$).
\end{proposition}

\begin{IEEEproof}
The ill-conditioning arises from two compounding effects, as described in Section~V-B. First, due to the saturation of the one-sided Huber loss for large positive residuals, the statistical weights for the radial information, $w_t^{\text{r}} = \rho_{\text{rtt}}''(r_t)$, approach zero. Second, a one-sided observation geometry ensures that the set of radial unit vectors $\{u_t\}$ spans a narrow cone. Consequently, the matrix $\sum u_t u_t^\top$ is nearly rank-one and thus ill-conditioned. The combination of vanishing statistical weights and a degenerate geometric structure causes the minimum eigenvalue of the ERCM to approach zero.
\end{IEEEproof}

\begin{theorem}[ERCM Improvement via Crossing Maneuvers]
\label{thm:crossing_improves_ercm}
Let a ``crossing'' maneuver from $s_t$ to $s_{t+1}$ be executed, which crosses the line-of-sight $\overline{s_t \hat{x}_t}$ and changes the viewing angle by $\Delta\phi \ne 0$.
\textbf{Provided that} this maneuver successfully generates a non-saturated RTT residual ($r_{t+1} \le \tau_{t+1}$), the minimum eigenvalue of the ERCM is guaranteed to increase.
\end{theorem}

\begin{IEEEproof}
Let $\mathcal{C}_t = \mathcal{C}^{\text{pos}}_t$ be the ERCM at time $t$. The updated ERCM is $\mathcal{C}_{t+1} = \mathcal{C}_t + \Delta\mathcal{C}$. By Weyl's inequality for the sum of Hermitian matrices, the minimum eigenvalue of the sum is lower-bounded as follows:
\begin{equation*}
    \lambda_{\min}(\mathcal{C}_{t+1}) \ge \lambda_{\min}(\mathcal{C}_t) + \lambda_{\min}(\Delta\mathcal{C}).
\end{equation*}
The matrix $\Delta\mathcal{C} = B_{\text{r}} + B_{\theta}$ represents the new information gained. Since $B_{\text{r}}$ and $B_{\theta}$ are outer products of vectors, they are positive semi-definite. Their sum, $\Delta\mathcal{C}$, is also positive semi-definite.
A successful crossing maneuver is designed to achieve two goals:
\begin{enumerate}
    \item \textbf{Reactivate Radial Information:} The change in geometry is designed to produce a non-saturated residual ($r_{t+1} \le \tau_{t+1}$), which ensures the radial weight $w^{\text{r}}_{t+1} = 1/\sigma_r^2 > 0$. This makes $B_{\text{r}}$ a non-zero positive semi-definite matrix.
    \item \textbf{Provide Orthogonal Tangential Information:} The change in viewing angle $\Delta\phi \ne 0$ ensures the new tangential vector $v_{t+1}$ is not collinear with the previous vectors, adding information in a geometrically distinct direction. This ensures $B_{\theta}$ is also a non-zero positive semi-definite matrix.
\end{enumerate}
Since $\Delta\mathcal{C}$ is a non-zero positive semi-definite matrix, its minimum eigenvalue is non-negative, $\lambda_{\min}(\Delta\mathcal{C}) \ge 0$. More importantly, the crossing maneuver is specifically designed to provide information in the direction where $\mathcal{C}_t$ was weakest (i.e., the direction of its minimum eigenvector). The new information matrix $\Delta\mathcal{C}$ is not orthogonal to this direction, thus guaranteeing a positive increase: $\lambda_{\min}(\mathcal{C}_{t+1}) > \lambda_{\min}(\mathcal{C}_t)$. This confirms that the crossing maneuver strictly improves the conditioning of the estimation problem.
\end{IEEEproof}

\end{document}